%% file: Style.tex
\documentclass{article}
\usepackage[utf8]{inputenc}
\usepackage[utf8]{inputenc}
\usepackage{authblk}
\input{own_commands.tex}

\title{%Long term Dynamics of Multiple Beliefs on Social Networks
Evolution of beliefs in social networks
\thanks{Correspondence may be addressed to \texttt{pparanamana@saintmarys.edu}.} 
}

\author{%
  Pushpi Paranamana$^1$, Pei Wang$^2$ \& Patrick Shafto$^{2,3}$\\
  
  $^1$Saint Mary's College, Notre Dame\\
  $^2$ Rutgers University--Newark\\
  $^3$ Institute for Advanced Study, Princeton\\
 }
\usepackage{natbib}
\usepackage{graphicx}

\begin{document}

\date{}
\maketitle

\begin{abstract}
    Evolution of beliefs of a society are a product of interactions between people (horizontal transmission) in the society over generations (vertical transmission). 
    Researchers have studied both horizontal and vertical transmission separately. Extending prior work, we propose a new theoretical framework which 
    allows application of tools from Markov chain theory to the analysis of belief evolution via horizontal and vertical transmission. We analyze three cases: static network, randomly changing network, and homophily-based dynamic network. Whereas the former two assume network structure is independent of beliefs, the latter assumes that people tend to communicate with those who have similar beliefs. We prove under general conditions that both static and randomly changing networks converge to a single set of beliefs among all individuals along with the rate of convergence. We prove that homophily-based network structures do not in general converge to a single set of beliefs shared by all and prove lower bounds on the number of different limiting beliefs as a function of initial beliefs. We conclude by discussing implications for prior theories and directions for future work. 
\end{abstract}

\input{main.tex}

\end{document}

%% file: own_commands.tex
\usepackage{color,soul}
\usepackage{amsmath}
\usepackage{amsthm}
\usepackage{amstext}
\usepackage{mathrsfs}
\usepackage{amssymb}
\usepackage{amsfonts}
\usepackage{bm}
\usepackage{blkarray}
\usepackage[utf8]{inputenc}
\usepackage{chngcntr}
\usepackage[ruled,vlined]{algorithm2e}
\usepackage{apptools}
\AtAppendix{\counterwithin{thm}{section}}
\usepackage{amsmath,tikz}
\usetikzlibrary{matrix}
\usetikzlibrary{automata}
\usetikzlibrary{arrows}
\usetikzlibrary{intersections}
\usepackage{fullpage}
\usepackage{float}
\usepackage{subfigure}% in preamble
\usepackage{comment}
\usepackage{float}
\usepackage{graphicx}
\usepackage{caption}
\usepackage{ulem}
\usepackage{tikz}
\usetikzlibrary{shapes}
\usepackage{booktabs}

\newcommand{\pw}[1]{\textcolor[rgb]{0.6,0.6,0}{#1}}

\newcommand{\red}[1]{\textcolor[rgb]{1,0,0}{#1}}

\specialcomment{Notes}{%
  \begingroup\color{red}}{%
  \endgroup}

\newtheorem{thm}{Theorem}

\newtheorem{cor}[thm]{Corollary}
\newtheorem{prop}[thm]{Proposition}

\theoremstyle{definition}
\newtheorem{definition}[thm]{Definition}

\newtheorem{remark}[thm]{Remark}
\newtheorem{example}[thm]{Example}

\usepackage{bm}
\usepackage{enumerate}

\newcommand{\s}{\mathcal S}

\newcommand{\avg}{\mathop{{}\mathbb{E}}}

\newcommand{\e}{\epsilon}

 \usepackage[numbers]{natbib}

%% file: main.tex
\section{Introduction}

Evolution of beliefs, individual and cultural, is the result of vertical transmission between generations and horizontal transmission within a generation.
%interactions between innate cognitive constraints on learning, direct experience, and social dynamics. 
Research in cognitive science has developed models of vertical transmission, through connections to probabilistic models of cognition \cite{chater2006probabilistic} and used such models to investigate innate cognitive constraints and connections to experience %\pat{cite tom griffiths}
\cite{griffiths2007language, kirby2007innateness}. Separately, research in network science has developed theories that explain horizontal transmission, the social dynamics of transmission and diffusion patterns \cite{newman2003structure, zhou2020realistic}. 
Because beliefs are shaped both by vertical and horizontal transmission, any successful theory of evolution of beliefs will need to combine aspects of both approaches. We propose a mathematical approach that enables detailed analysis of the long run consequences of vertical and horizontal transmission for individual and cultural beliefs. 

Theories in cognitive science frame vertical transmission through evolution as functional adaptations of cognitive capacities, such as language, beliefs, knowledge and metacognition, to ancestral environment. %\pat{cite tom!} 
\cite{griffiths2007language, whalen2017adding, kirby2007innateness, suchow2017evolution} have developed methods to interpret vertical transmission between Bayesian agents as Markov chains, thus revealing innate cognitive constraints and structures as the outcome of such processes. For example, \cite{griffiths2007language} 
%\cite{kirby2007innateness}, if strong learning biases must be maintained against mutation pressure, the introduction of cultural transmission may lead to a weakening of these innate biases.-something like this? -(should rephrase)
interpret transmission of language from parents to children as a Markov chain, which leads to the conclusion that, in the absence of other influences, the resulting observed distribution of languages reflects our prior biases about language and language structures. 

However, cognition and memory are sustained by both communicative and cultural aspects \cite{candia2019universal, zhou2020realistic} and reflect social influences \cite{roberts2018social, abrams2011dynamics}. This horizontal transmission is intrinsically bidirectional and introduces the possibility of long term consequences of social network structures for beliefs. Network theory has studied transmission over social networks \cite{boccaletti2006complex, delvenne2015diffusion} for cases including diseases \cite{pastor2015epidemic, huang2019epirank}, information \cite{zhou2020realistic, zhan2019information, wang2013characterizing}, opinions \cite{castellano2012social, quattrociocchi2014opinion}, and rumours \cite{moreno2004dynamics}. % \pat{adjust this list as necessary!}.
However, in these models transmission is formalized as a property that can be caught or passed between agents. This is suitable for diseases and facts, but beliefs are more naturally represented as distributions over some latent space, as in probabilistic models of cognition used to model vertical transmission.

In this article, we combine both vertical and horizontal transmission to explore the long term evolution of beliefs. We provide a mathematical formulation to analyze the limiting distribution of beliefs in societies based on sociodynamic aspects and cognitive aspects of belief evolution. This limiting distribution tells us the long term belief distribution of each individual. Moreover this provides a framework to explore the long term belief evolution of groups and/ or of the society as a whole. Integrating classical results of time homogeneous and inhomegeneous Markov chain theories, we provide conditions on the network structures--static and dynamic at random--that result in homogeneous/hetrogeneous belief systems among individuals (or groups). Moreover, we provide rates of convergence of the models to their limiting behaviors, for both static and random cases. Prior studies show that individuals in a social network may tend to connect to individuals who share similar interests, and thus it is considered as an important evolutionary mechanism \cite{liu2018competition}.
We integrate this assortive dynamics in which networks are formed based on homophily and prove conditions under which societies will converge to heterogenous beliefs.

% \red{XXXXXXXXXXXXXXXXXXXXXXXXX}

% \pu{ brief literature review of how the quality of collective intelligence is influenced by belief diversity.connecting belief structure to collective intelligence or explaining how the results of the paper might impact our understanding of the conditions supporting collective intelligence and explaining how the results of the paper might impact our understanding of the conditions supporting collective intelligence.

There has been extensive research on how belief diversity enhances the collective intelligence. 
%Collective intelligence depends on existence or lack of differences in individual beliefs. That is the society collectively has similar beliefs and lacks diversity individuals in the society is likely to obtain similar decisions, giving less space for collective intelligence. 
A society that collectively has similar beliefs offers little chance for collective decision making to improve over any individuals. 
If individuals have different beliefs, collective accuracy can be enhanced. 
A simple example comes from ``wisdom of crowds'' effects in which the average of a group of people's guesses is more accurate than most individuals \cite{galton1907vox}, but many more examples exist in the decision making literature. 
Integration of multiple beliefs and, diversity in beliefs is thus required for underlying collective intelligence \cite{novaes2018individuals, keuschnigg2017crowd,broomell2009experts}. In this work we explore the network and belief structures that result in belief homogeneity vs heterogeneity under three scenarios: static networks, randomly changing networks and homophily-based networks. Thus the results can be used to explore conditions on optimal structures that improves collective accuracy and evolution.

\section{Formulation of the problem} \label{Sec:Formulation}

Our aim is to develop a model that one can use to analyze evolution of individual and societal beliefs through both vertical and horizontal transmission.
%
%\pw{maybe move the next two sentences to the analysis section or introduction?}
%Specifically, we seek to address basic questions regarding when beliefs within society will converge to a single stable point, and when beliefs will diverge to multiple groups. 
%We seek to understand such homogeneity or heterogeneity in terms of the initial beliefs of individuals, confusion or decay of beliefs over time, the structure of social interactions, and changes in social interactions over time. 
%
Our approach builds on prior research in the cognitive science literature formalizing vertical transmission as a Markov Chain \cite{kirby2007innateness, whalen2017adding, griffiths2007language}, while integrating horizontal transmission from network theory.

To integrate horizontal transmission, we formalize interactions among individuals in a society with a given, possibly dynamic, structure. As in prior work, individuals' initial beliefs are assumed to be sampled from a given distribution. Individuals within a society will interact with subsets of other individuals as defined by an adjacency matrix defining network structure. Networks may take a variety of forms including unidirectional and bidirectional, static and dynamic, and belief dependent. 
%For the subsequent analyses, these distinctions will not be material because all can be represented as an adjacency matrix. 
Each of these cases can be represented as a (collection of) adjacency matrix (matrices).

\begin{definition} \textbf{Evolution of beliefs in social networks.}
Consider a set of people $\mathcal{P}=\{{\alpha_i}\}_{i=1}^r$ in the society and a set of concepts $\mathcal{H}=\{\beta_k\}_{k=1}^s$.
Denote people's priors on $\mathcal{H}$ by $M=\left(m_{jk}\right)_{r\times s}$, %, where $m_{jk}$ is the initial belief (prior) of the person $\alpha_j$ on the belief $\beta_k$, 
the network structure over which people may communicate at time $t$ by $P_t=\left(p_{ij}\right)_{r\times r}$, %, where $p_{ij}$ is the weight that the $i^{th}$ person  gives to the $j^{th}$ person's information 
and the concept structure at time $t$ by $H_t=\left(h_{kl}\right)_{s \times s}$. All are row stochastic matrices. %, where $h_{kl}$ denotes the %amount of understanding probability of realizing $\beta_l$ given $\beta_k$
Let $P_0=H_0=I$, where $I$ is the identity matrix of corresponding order.
% \pw{H and P has different sizes?}
 Define  $$ Q_n(P_t, H_t, M)=\prod_{t=0}^{t=n} P_t\, M \prod_{t=0}^{t=n} H_{t}, $$ \label{eq:model} 
% \pw{consider switch to
% $$ Q_n(P_t, H_t, M)=P_n \dots, P_0 M H_0,\dots H_n, $$}
where $Q_n$ represents the society's beliefs at time $n$, and the long-run beliefs are analyzed as $n~\rightarrow\infty$. That is, at each time $n$, $Q_n$ is the product $P_n \dots, P_0 M H_n,\dots H_0$ . 
% \pu{(Reviewer's suggestion is to do $$ Q_n(P_{t_1}, H_{t_2}, M)=\prod_{{t_1}=0}^{{t_1}=n} P_{t_1}\, M \prod_{{t_2}=0}^{{t_2}=n} H_{t_2}, )$$}
% \pat{So which are we proposing to use? I am confused.} 
% \label{eq:model} 
% $$Q_n(M)=\prod_n P_n\, M \prod_n H_n $$ 
% and $\mathcal{I}_n(\alpha_i,\beta_j)=Q_n(M).$ %=P^n \, I \, M \, (S H)^n $
%Let $\mathcal{I}_n(\alpha_i,\beta_j)$ be the $ij$-th element of $Q_n(P_t,H_t,M)$. Then, $\mathcal{I} _t(\alpha_i,\beta_j)$ is the amount that the person $\alpha_i$ believes concept $\beta_j$, at any discrete time $t$.
%\pu{write as a definition?}
\end{definition}
%%%%%%%%%%%%%%%%%%%%%%%%%%%%%%%%%%
\noindent
The model formulates the time evolution of people's beliefs. $m_{jk}$ represents the initial belief (prior) of the person $\alpha_j$ on the concept $\beta_k$, 
$p_{ij}$ denotes the weight that the $i^{th}$ person gives to the $j^{th}$ person's information and $h_{kl}$ denotes the %amount of understanding 
degree to which concept $\beta_l$ may be confused for $\beta_k$. 
%\pu{better wording instead of probability?}. 
A variety of properties can be captured in the matrices $P_t$ and $H_t$. Consider $P_t$. Absence of direct transmission is formalized when $p_{ij}=p_{ji}=0$. Bidirectional transmission is formalized by $p_{ij}>0$ and $p_{ji}>0$. Unidirectional transmission is formalized when either $p_{ij}>0$ and $p_{ji}=0$ or $p_{ij}=0$ and $p_{ji}>0$. Different network structures including random graphs, small world and scale-free networks \cite{albert2002statistical} can be formalized through the construction of adjacencies. Dynamic network structures \cite{li2017fundamental} are formalized by introducing the subscripts $P_t$ and $H_t$ to indicate the network structure at time $t$.
%Moreover, the structure of networks may be dependent upon other factors, such as the similarity of beliefs between individuals in $M$. In general, we model the matrix $M$ using Dirichlet distribution and assume that it is indecomposable. However, similar analysis follows when $M$ is decomposable and is discussed in Remark \ref{sec:block M}.
Notice that $P_t, M$ and $H_t$ are stochastic matrices. Therefore for any $n$, $Q_n(P_t,H_t,M)$ is also a stochastic matrix.
This approach combines both the sociodynamic and the cognitive aspects of belief evolution which helps to evaluate society's vertical and horizontal transmission simultaneously. 

\begin{table}[H]
 \caption{Model Summary}
    \label{tab:unifying}

    \centering
    %\begin{tabular}{ |p{4cm}||p{2cm}|p{2cm}|p{2cm}|  }
    %\begin{tabular}{ |c||c|c|c|  }
%  \hline
%  \multicolumn{5}{|c|}{Country List} \\

\begin{tabular}{ ll }
 \toprule
Notation  & Definition\\
 \midrule
$\mathcal{P}=\{{\alpha_i}\}_{i=1}^r$ & a set of people in the society, where $\alpha_i$ denotes the $i^{th}$ person\\
$\mathcal{H}=\{\beta_k\}_{k=1}^s$ & a set of concepts, where $\beta_k$ denotes the $k^{th}$ concept  \\
$M=\left(m_{jk}\right)_{r\times s}$ & a row stochastic matrix records a set of people's priors over a set of concepts \\
 & each row represents a person, each column represents a concept\\
 & $m_{jk}$ denotes the initial belief of person $\alpha_j$ on concept $\beta_k$\\
$P_t=\left(p_{ij}\right)_{r\times r}$ & a row stochastic matrix denotes the network structure at time $t$\\
&  $p_{ij}$ denotes the weight that $i^{th}$ person gives to $j^{th}$ person's information \\
$H_t=\left(h_{kl}\right)_{s \times s}$ & a row stochastic matrix denotes the concept structure at time $t$\\
& $h_{kl}$ the degree to which concept $\beta_l$ may be confused for $\beta_k$\\
$ Q_n(P_t, H_t, M)$ & a row stochastic matrix modeled by $\prod_{t=0}^{t=n} P_t\, M \prod_{t=0}^{t=n} H_{t}$\\
& records the society's beliefs at time $n$\\
\bottomrule
\end{tabular}
    \vspace{1pc}
   
\end{table}

Next, we illustrate the design of the structures and the model using some stylized examples.
%In this work, we assume that people's lapse of memory is negligible. This is reasonable since the people's communication matrix and the belief confusion matrix appear in the model at each time step. \pw{check, not following the logic here...}

% \begin{remark} \label{Remark:stoc}
% Notice that, $P_t, M$ and $H_t$ are stochastic matrices. Therefore for any $n$, $Q_n(P_t,H_t,M)$ is also a stochastic matrix.
% % \pw{add a little proof?}

% \end{remark}

\begin{example}
Consider a neighborhood with three people $\alpha_1, \alpha_2, \alpha_3$ with the network structure at time $t$ given by 
\[ P_t=
\begin{blockarray}{cccc}
&\alpha_1 & \alpha_2 & \alpha_3  \\
\begin{block}{c(ccc)}
 \alpha_1 & 1 & 0 & 0 \\
 \alpha_2 & \frac{2}{3} & 0  & \frac{1}{3} \\
 \alpha_3 & \frac{1}{2} & \frac{1}{4} & \frac{1}{4} \\
\end{block}
\end{blockarray}
\]
The network structure can be depicted in Figure~\ref{NS}, where the directed edge from $\alpha_i$ to $\alpha_j$ denotes $p_{ij}$.

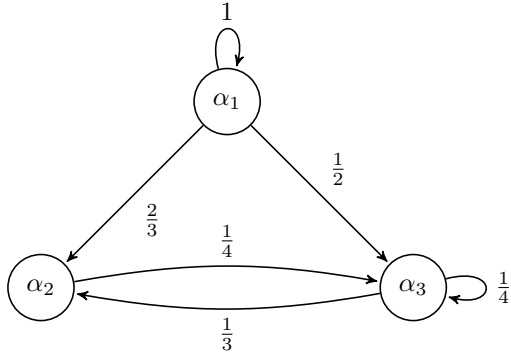
\begin{figure}[H]
\begin{center}
\begin{tikzpicture}[->,>=stealth',shorten >=1pt,auto,
  node distance=3.5cm,semithick, bend angle=10]
\tikzstyle{every state}=[fill=none,draw=black,text=black,shape=circle]
\node[state] (x1) {$\alpha_1$};
\node[state] (x3) [below right of =x1] {$\alpha_3$};
\node[state] (x2) [below left of =x1] {$\alpha_2$};
% \draw[name path=e1] (x1) to [bend left] node {$0$} (x2);
\draw[name path=e2] (x1) to node {$\frac{2}{3}$} (x2);
\draw[name path=e3] (x1) to node {$\frac{1}{2}$} (x3);
% \draw[name path=e4] (x1) to [bend left] node {$0$} (x3);
\draw[name path=e5] (x2) to [bend left] node {$\frac{1}{4}$} (x3);
\draw[name path=e6] (x3) to [bend left] node {$\frac{1}{3}$} (x2);
\draw(x1) to [loop above] node {$1$} (x1);
% \draw(x2) to [loop left] node {$0$} (x2);
\draw(x3) to [loop right] node {$\frac{1}{4}$} (x3);
\end{tikzpicture}
\caption{Network structure} \label{NS}
\end{center}
\end{figure}

Here, $\alpha_1$ does not believe what anyone else says but believes himself $100\%$, while $\alpha_2$ only believes what others say. However, $\alpha_3$ believes him self and others with certain percentages. In practice, this may model the communication between three listeners $\alpha_1, \alpha_2, \alpha_3$, where $\alpha_1$ may be a speaker,  $\alpha_2$ may model a new student in the class who is learning from his teacher $\alpha_1$ and a peer $\alpha_3$.
\end{example}

Similarly, the formulation of the concept structure can be viewed as a graph. Instead of pointing from speaker to listener, arrows point from a concept toward a concept it can replace (be confused with).
Note that the concept structure is modeled as a distribution rather than single values. Each value corresponds to the degree (weight) to which one concept may be confused with another. Notice that the concept structure corresponds to the vertical component of the model. That is, the generational or cultural transmission of beliefs. Confusability of beliefs is a reasonable notion to denote the imperfect dynamics over generations due to changes in cultural traits and new found information over time leading to generation gaps.

% \pu{In the new paragraph on the concept structure (line 23), it would be helpful to more
% explicitly highlight this as the vertical component of the model. It would also be helpful to
% add a sentence or so on why something like confusability is a reasonable/practical way
% to model imperfect transmission over generations/iterations.}

% \begin{example}
% \red{
% Consider a group of $10$ people, each holding a belief on $5$ distinct concepts. This prior belief distribution is given by $M$ which is a $10\times 5$ row stochastic matrix. 
% At every point in time each individual communicates with other individuals in the group according to the network structure. Moreover, concept structure gives the confusion between concepts at a particular time step. Based on these interactions, each individual are influenced to alter their beliefs based on the beliefs of those whom they communicated with and the confusability between concepts. $Q_n$ represents the beliefs of each individual at time $n$. By looking at how the structure of $Q_n$ changes as $n$ changes, not only we can explore how individual beliefs changes but also how societal beliefs evolve over time.
% }
% \end{example}

\begin{example}\label{stylizedEx} 
Consider a group of $5$ people, each holding a belief on $5$ distinct concepts.
Suppose people have a prior belief distribution given by and $M= I_{5\times 5}$. That is $\alpha_i$ believes only on the concept $\beta_i$, for all $i=1,..,5.$ 
Let
$P=\small{
\begin{pmatrix}
0.194 & 0.387 & 0.419 & 0  & 0   \\
0.29 & 0.323 & 0.387 & 0   &  0   \\
0.261 & 0.696 & 0.043 & 0   & 0   \\
0   &  0   & 0   & 0.448 &  0.552\\
0   & 0   & 0   & 0.2  & 0.8  
\end{pmatrix}}$ be the network structure and $H=\small{
\begin{pmatrix}
0.342 & 0.421 & 0.026 & 0.105 & 0.105 \\
0.163 & 0.204 & 0.388 & 0.102 & 0.143\\
0.289 & 0.044 & 0.156 & 0.178 & 0.333 \\
0.316 & 0.105 & 0.246 & 0.158 & 0.175\\
0.304 & 0.107 & 0.286 & 0.286 & 0.018
\end{pmatrix}}$ be the concept structure. Assume that both the network structure and the concept structure are constant over time. 
That is, at every point in time each individual communicate with others in the society according to the network structure $P$. Notice that the society consists of two groups that do not talk to the other group. Similarly the concept structure at each time step, which gives the confusion between concepts at that particular time step, does not change over time. Based on these interactions, each individuals are influenced to alter their beliefs based on the beliefs of those whom they communicated with and the confusability between concepts.
We can use the model to explore the belief evolution over time. In particular, $Q_n$ represents the belief distribution of the society at time $n$, where its $i^{th}$ row denotes the belief distribution of the $i^{th}$ person at time $n$. By looking at how the structure of $Q_n$ changes as $n$ changes, not only we can explore how individual beliefs changes but also how societal beliefs evolve over time. Moreover, analyzing $Q_n$ as $n\to \infty$, we can explore the societal long term belief distribution. In this example, we can show that the society stabilizes to a homogeneous belief distribution where all the individuals in the society have the same beliefs. (Please see Example~\ref{Ex2} for detailed analysis). It is interesting to see the homogeneity of the society even though there are two groups that do not talk to the other group at all.
% }
\end{example} 

In the next section, we analyze the long term behavior of the model theoretically which sheds light on belief evolution and societal belief diversity. The above example illustrates the model for a static network structure and a static concept structure. However the structures could be dynamic, thus we will consider three phenomena: static structures, randomly changing structures and homophily-based dynamic structures.
Analyzing this model will help us better understand the minimal conditions necessary for sustained belief heterogeneity, conditions on which the homogeneity is attained.

\noindent Note: Markov chains are widely used in many
applications in predicting variation tendencies of random processes including modeling inter generational beliefs. Belief evolution can be studied as transmission chains where the beliefs evolve through time via horizontal and vertical transmission, which is mathematically parallel to analyzing Markov
chains. So in our model both the network structure and the concept structure are considered as Markov chains and are represented by corresponding transition matrices.

\section{Analyzing the belief evolution in social networks}\label{sec2}
%Analysis and Simulations} 
%\pu{analysis for time homogeneous, inhomogeneous-long term behavior, rate of convergence}

In this section, we explore belief change in the long run, individually and societally. %We discuss the structure of the limit depending on the social and belief structures. 
We analyze under what conditions a society will attain homogeneity of beliefs and whether the society will evolve into groups with distinct beliefs. Moreover, we explore how fast a society will converge to its final belief system.
As discussed in Section~\ref{Sec:Formulation}, networks can be time invariant as well as time variant. Therefore we investigate the belief evolution for time homogeneous and time inhomogeneous cases separately. 

\subsection{Belief evolution over stable social and belief networks}\label{sec:homo_time}

First, we analyze the belief evolution when network and concept structures are time homogeneous. %while the confusion among beliefs remain constant over time. 
That is, we assume that $\forall t, P_t=P$ and $H_t=H$; $P$ and $H$ are fixed matrices. Then the operator $Q$ simplifies to 
$$Q_{n}(P,H,M)=P^n\,M \,H^n.$$
For a square matrix $A$, $A^n$ denotes the multiplication of $A$ for $n$ times. 
%\subsubsection{Convergence and the stationary distribution}
%We first analyze the convergence of a time homogeneous Markov chain with transition matrix $P$ (similarly for $H$). The convergence results of Markov chains presented below are not new. However, we use them to analyze the limit of the model and thus will summarize important preliminary results below (See \cite{gravner2010lecture} for more details). 
\footnote{We use Markov chain and corresponding transition matrix interchangeably, when there is no confusion.} 

\subsubsection{Convergence and limiting distribution} \label{sec:limit time homo}
As transition matrices of Markov Chains, important distinctions about the network and concept structure are whether they are indecomposable/decomposable and reducible/irreducible. The limiting behavior depends on the structures as well as the states of the people and beliefs (transient/persistent). Therefore we define: %(See Definitions \ref{def:indecom} and \ref{def:irred})
% \pat{name definitions?}

\begin{definition}\textbf{Irreducibility:} \label{Def:closed,irred,indecom}
A set $C$ of states is closed if no state outside $C$ can be reached from any state %$X_j$ 
$j$ in $C$.
A Markov chain is irreducible if there exists no closed sets other than the set of all sets; otherwise, it is reducible.
\end{definition}

\begin{definition}\textbf{Indecomposability:} \label{def:indecom}
A Markov chain is indecomposable if it contains at most one closed set of states other than the set of all states. Otherwise it is decomposable.
\end{definition}

\begin{definition}\textbf{Transient/Recurrent states:} \label{def: transient,persistent}
%Hitting time of state $i$ is the first return time to state $i$, namely,
%$T_i:=\inf\{n\ge 1: X_n=i\}$.
State $i$ is called transient if, given that we start in state $i$, there is a non-zero probability that we will never return to $i$. %That is, state $i$ is transient if $ \Pr(T_{i}<\infty \mid X_{0}=i)<1.$
State $i$ is called recurrent (or persistent) if it is not transient. %Recurrent states are guaranteed (with probability 1) to have a finite hitting time. 
% State $i$ is called absorbing if and only if
% $ p_{ii}=1{\text{ and }}p_{ij}=0{\text{ for }}i\not =j.$ That is, it is impossible to leave that state. 
\end{definition}

\begin{definition}\textbf{Stationary distribution:}
Let $A$ be a transition matrix.
A stationary distribution (steady state distribution) $\bm{\pi}$ is a non negative stochastic (row) vector, %whose entries are non-negative and sum to $1$ 
such that $\bm{\pi} A =\bm{\pi}.$
\end{definition}

%\pw{def of SIA should be defined somewhere here.} 

% \begin{prop} \label{prop:same rows}
% \pw{ref?}
% If a time homogeneous finite Markov chain is indecomposable and aperiodic, then there exists a unique stationary distribution $\bm \pi$ such that $$\lim_{n \to \infty} A^n = \bm 1 \bm \pi$$ where $\bm 1$ is the column vector with all entries equal to 1.
% \end{prop}
%Proposition \ref{prop:same rows} shows that 
\noindent
An indecomposable and aperiodic (Definition \ref{def:period}) markov chain has a unique stationary distribution $\bm \pi$ as $n\to \infty$ \cite{gravner2010lecture}. That is, the transition matrix converges to a matrix with same rows equals to $\bm \pi$.
%That is, it converges to a rank one matrix in which each row is $\bm \pi$. 
Moreover, $\bm \pi$ is the left eigenvector of the associated transition matrix that corresponds to the unit eigenvalue (which exists and is unique). 
%Thus, if the network structure $P$ (also, concept structure $H$) is indecomposable and aperiodic, then the society will converge to a single belief distribution in the long run.
%Moreover, %using Lemma \ref{Lemma} we can see that 
If the Markov chain is indecomposable but reducible, the transient states vanish in the limit. 
%This means that in the long run, everyone in the society believe (with same weights) the beliefs of people who are persistent, as explained in Example \ref{Ex:transient people}.
Similarly, we can analyze the structure of the limit of decomposable, aperiodic chains using 
Propositions \ref{prop:period} and \ref{prop:decom,aperiod}. 
%\pw{may omit?}

%\vspace{2mm}
%\noindent \textbf{Limiting distribution}

% \hl{change confusability to concept structures, also belief structures to concept sturctures} 
%Next, we analyze the long term dynamics of the model $Q_n(P, H, M)$. 
%Now, we look at the structure of the limit of the model under different conditions on network and concept structures.
%In the following, we assume that all the matrices are aperiodic. Otherwise, similar analysis can be performed using propositions \ref{prop:period} and \ref{prop:decom,aperiod}.
%Each matrix $P$ and $H$ can be either indecomposable or deomposable. 
%Recall that $M$ is modeled using dirichlet distribution and assumed to be indecomposable. %We perform the analysis by dividing into two cases, namely, case 1: when $H$ is indecomposible, and case 2: when $H$ is decomposible.\\
% We categorize matrices ($P$ and $H$) in to two classes: class 1: indecomposable , class 2:  decomposable markov chains

\begin{prop} \label{prop:homogeneous} \footnote{All proofs are included in the Supplemental Material.}
%\pw{aperiodicity is missing?}
Assume the network structure $P_{r\times r}$ and the concept structure $H_{s\times s}$ are aperiodic matrices. Let $M_{r\times s}$ be the initial belief distribution in the society. The society will stabilize in the long run. That is $\|Q_{n+1}-Q_n\|\to 0$ as $n\to \infty$.
%$lim_{n\to \infty} Q^{n+1}=lim_{n\to \infty} Q^{n}$. Moreover,

\begin{enumerate} [(i)]
    \item \label{item1}
If $H$ is indecomposable, then in the long run, the society stabilizes to a single belief distribution that does not depend on $P$ or $M$. 
That is, there exists a steady state distribution $\bm{\pi}=\{\pi_1,....,\pi_s\}$ 
% (which is a left eigenvector of $H$ with eigenvalue equals $1$)
such that for any $P$ and $M$
$$\lim_{n\to \infty}Q_{n}=\lim_{n\to \infty}P^nMH^n=
\begin{pmatrix}
\pi_1 & ...& \pi_s\\
\vdots & \ddots & \vdots \\
\pi_1 & ...& \pi_s\\
\end{pmatrix}_{(r\times s)}
%=lim_{n\to \infty}H^n
$$

\item \label{item2}
If $H$ is decomposable and 
$P$ is indecomposable, then in the long run, the society stabilizes to single a belief distribution that depends on $M$ and $H$.
That is for any $M$ and $H$, there exists a steady state distribution $\bm{\sigma}=\{\sigma_1,\sigma_2,...,\sigma_s\}$ such that:
$$\lim_{n\to \infty}Q_{n}=\lim_{n\to \infty}P^nMH^n=
\begin{pmatrix}
\sigma_1 & ...& \sigma_s\\
\vdots & \ddots & \vdots \\
\sigma_1 & ...& \sigma_s\\
\end{pmatrix}_{(r\times s)}
%=lim_{n\to \infty}H^n
$$

\item \label{item3}
If $H$ and $P$ are both decomposable. Then the society will not have a single belief distribution in the long run.
That is, the rows of the matrix $\lim_{n\to \infty} Q_{n}$ are not all the same. 
%The number of stable groups with same belief distribution depends on the state of the people (transient or persistent). 

%\pw{can we be more precise in this case?} \pat{agree}
%\pu{I do not know particular results/formulas for this case. Different scenarios are discussed in examples.}

%Proof \ref{proof:homogeneous} \pu{fix}
%Proof can be easily obtained by Proposition \ref{prop:decom,aperiod}.
\end{enumerate}
\end{prop}

Notice that, if $H$ is indecomposable neither $P$ or $M$  have an effect on the limit of the model. That is, the concept structure dominates and controls the long run behavior, regardless of what the network structure is or what people initially believe. Moreover, as a Markov chain, transient beliefs (if there are any) vanish from the society. However if $H$ is decomposable, in addition to the concept structure, the network structure as well as the initial concept structure affect the long run behavior. The homogeneity of beliefs among people in the society depends on the network structure. In particular, if the  
network structure is indecomposable, then there will be a unique belief distribution in the society regardless of the initial beliefs. In summary, if either the network structure or the concept structure is indecomposable, the society will converge to a unique belief distribution in the long run. However, if both are decomposable, there will not be a single belief distribution in the society; there will be heterogeneity among individuals. These scenarios are illustrated in Examples %\ref{Ex1}, 
\ref{Ex2}, \ref{Ex3}, \ref{Ex4}. 

\begin{example}\label{Ex2} 
% Let
% $P=\small{
% \begin{pmatrix}
% 0.194 & 0.387 & 0.419 & 0  & 0   \\
% 0.29 & 0.323 & 0.387 & 0   &  0   \\
% 0.261 & 0.696 & 0.043 & 0   & 0   \\
% 0   &  0   & 0   & 0.448 &  0.552\\
% 0   & 0   & 0   & 0.2  & 0.8  
% \end{pmatrix}},\,\,H=\small{
% \begin{pmatrix}
% 0.342 & 0.421 & 0.026 & 0.105 & 0.105 \\
% 0.163 & 0.204 & 0.388 & 0.102 & 0.143\\
% 0.289 & 0.044 & 0.156 & 0.178 & 0.333 \\
% 0.316 & 0.105 & 0.246 & 0.158 & 0.175\\
% 0.304 & 0.107 & 0.286 & 0.286 & 0.018
% \end{pmatrix}},$
% and $M= I_{5\times 5}. \,\, $
Consider $P, M$ and $H$ given in Example~\ref{stylizedEx}.
Here, $P$ is decomposable and has two closed communicating classes and $H$ is indecomposable.
Then, $\lim_{n\to \infty}Q_n= \small{
\begin{pmatrix}
0.285 & 0.203 & 0.2  & 0.155 & 0.156\\
0.285 & 0.203 & 0.2  & 0.155 & 0.156\\
0.285 & 0.203 & 0.2  & 0.155 & 0.156\\
0.285 & 0.203 & 0.2  & 0.155 & 0.156\\
0.285 & 0.203 & 0.2  & 0.155 & 0.156
\end{pmatrix}}
$
In this example, even though there are people in the network who never talk to each other, everyone converges to the same beliefs in the long run. This is because the concept structure $H$, which is indecomposable, dominates.
\end{example}

\begin{example} \label{Ex3}
Consider $
P= \small{
\begin{pmatrix}
0.3 & 0.7 & 0\\
0.6 & 0.4 & 0\\
0 & 0 & 1
\end{pmatrix}},
\,\,  M=\small{
\begin{pmatrix} 
0.1 & 0.3 & 0.12& 0.02 & 0.28 & 0.18\\
0.06& 0.23& 0.1 & 0.23& 0.23& 0.15\\
0.03& 0.01& 0.39& 0.38& 0.09& 0.1 
\end{pmatrix}},$
  and 
 %$$
% H=
% \begin{pmatrix}
% 1 & 0  & 0    & 0 & 0 & 0 \\
% 0.5  & 0  & 0.5 & 0  & 0 & 0 \\
% 0.25& 0  & 0.25 & 0.5 & 0  & 0 \\
% 0.125 & 0 & 0.125 &0.25  &0.5  & 0\\
% 0.0625 & 0  & 0.0625 & 0.125 &0.25 & 0.5  \\
% 0 &0 &0  &0 & 0 & 1 
%  \end{pmatrix},

 $
\,H=\small{
\begin{pmatrix}
0.803& 0.197& 0   & 0 &   0   & 0   \\
0.464& 0.536& 0&    0   & 0  & 0   \\
0&    0  &  0.272& 0.038& 0.02 & 0.669\\
0&    0   & 0.515& 0.017& 0.401& 0.068\\
0   & 0 & 0.144& 0.319& 0.002& 0.535\\
0 & 0 & 0.16 &0.357& 0.242& 0.241
\end{pmatrix}},
 $
where both $P$ and $H$ are decomposable. $P$ has two closed communicating calsses: $\{\alpha_1, \alpha_2\}$ and $\{\alpha_3\}$. Then the stationary distribution (rounded up to 3 decimal places) is
$\lim_{n\to \infty}Q_n=
% \begin{pmatrix}
% 0.4577& 0.    & 0.    & 0.    & 0.    & 0.5423\\
% 0.4577& 0.    & 0.    & 0.    & 0.    & 0.5423\\
% 0.442& 0.    & 0.    & 0.    & 0.    & 0.558
% \end{pmatrix}
\small{\begin{pmatrix}
0.239& 0.102& 0.169& 0.131& 0.115& 0.241\\
0.239& 0.102& 0.169& 0.131& 0.115& 0.241\\
0.028& 0.012& 0.245& 0.191& 0.167& 0.351
  \end{pmatrix} } .   
       $
 We can see that in the limit, people's beliefs are not the same. However, the beliefs of people who are in the same closed communicating class are the same. That is, $\alpha_1$ and $\alpha_2$ have the same beliefs while $\alpha_3$ has different beliefs.
\end{example}

\begin{example} \label{Ex4}
Consider
$P=
% \begin{pmatrix}
% 1 & 0.  & 0.    & 0 & 0 & 0 \\
% 0.5  & 0.  & 0.5 & 0.  & 0 & 0 \\
% 0.25& 0.  & 0.25 & 0.5 & 0.  & 0 \\
% 0.125 & 0. & 0.125 &0.25  &0.5  & 0\\
% 0.0625 & 0.  & 0.0625 & 0.125 &0.25 & 0.5  \\
% 0.&0. &0.  &0. & 0. & 1. 
%  \end{pmatrix},
\small{\begin{pmatrix}
1   & 0   & 0  &  0  &  0  &  0   \\
0.04 & 0.12 & 0.027& 0.027& 0.139& 0.647\\
0.052 & 0.182& 0.044& 0.519& 0.152& 0.051\\
0.006& 0.111& 0.032& 0.037& 0.729& 0.085\\
0.062& 0.416& 0.012& 0.38 & 0.123& 0.007\\
0  & 0    0 &  0 & 0 & 1
\end{pmatrix}},
\,\,  M=
\small{\begin{pmatrix}
0.268 & 0.422 & 0.31 \\
0.331 & 0.232 & 0.437 \\
0.094 & 0.364 & 0.542\\
0.172 & 0.559 & 0.268\\
0.239 & 0.74  & 0.021\\
0.048 & 0.011 & 0.941
\end{pmatrix}}$
and $
H=\small{
\begin{pmatrix}
0.3 & 0.7 & 0\\
0.6 & 0.4 & 0\\
0 & 0 & 1
\end{pmatrix}},
$
where both $P$ and $H$ are decomposable. In $P$, $\{\alpha_1\}$ and $ \{\alpha_6\}$ are closed classes. %and $\{\alpha_2, \alpha_3,\alpha_4,\alpha_5\}$ is the set of transient states. 
$\lim_{n\to \infty}Q_n=
% \begin{pmatrix}
% 0.3183 & 0.3714 & 0.3103\\
% 0.2601 & 0.3034 & 0.4365\\
% 0.2018 & 0.2355 & 0.5627\\
% 0.1436 & 0.1675 & 0.6889\\
% 0.0854 & 0.0996 & 0.8151\\
% 0.0271 & 0.0316 & 0.9412
% \end{pmatrix}
\small{\begin{pmatrix}
0.318 & 0.372& 0.31 \\
0.052& 0.061& 0.887\\
0.082& 0.096& 0.823\\
0.074& 0.087& 0.839\\
0.081& 0.094& 0.825\\
0.027& 0.032& 0.941
\end{pmatrix}}.
$ We can see that there will not be a unique belief distribution among people in the limit. %\pu{However if the prior is uniformly distributed, then the society will have a unique belief distribution}.

%\pw{It seems that $\{\alpha_2, \alpha_3,\alpha_4,\alpha_5\}$ are in different transient classes, not the same class.}
\end{example}
%%%%%%%%%%%%%%%%%%%%%

%What does it mean to say the network structure is indecomposable or decomposable, in practice? \pu{What does it mean to say the network structure is irreducible/ reducible/ indecomposable/ decomposable/ bidirectional/ unidirectional etc., in practice?}
\noindent \textbf{How are the social/concept structures represented by these different matrices?}
Social structures are typically highly structured. For example, some are strongly connected. That is, it is possible to communicate from any person by a chain of individuals to any other person in the network. This scenario can be represented by an indecomposable matrix. 
On the other hand there are social networks where the communication is unidirectional. For instance, media can be thought of as a unidirectional communication path in the sense that the news is broadcast, and no matter how loud one yells at the screen, the newscaster cannot hear you; hence, the audience's beliefs are transient. 
%\pw{so audiences' beliefs are transient, and will vanish? the structure here is decomposable, why?}. 
Also, some structures have a strong asymmetry between groups. Colonialism is such an example. 
%Overtime, beliefs and habits of settlers strongly affect the lifestyles of existing indigenous people. However, the change of beliefs of the settlers does not change much.
%Moreover, some communities might not communicate with other groups. 
These scenarios can be represented by different structures of decomposable matrices. Similar analogy can be made for concept structures based on the relatedness between beliefs.

%\begin{remark}\label{sec:block M}
%Notice that, if $M$ is decomposable, similar analysis can be performed as in Proposition \ref{prop:homogeneous}. %by replacing $P^n$ with $P^n M$. 
% \pw{check, not sure about the statement of this remark.} \pat{discuss}
% In particular, if $H$ is indecomposable, similar proof follows as in  Proposition \ref{prop:homogeneous} case $(i)$ indicating a unique stationary distribution that does not depend on $P$ or $M$. If $H$ is decomposable and $P$ is decomposable, society stabilizes to belief distribution that depends on $M$ and $H$, following the analysis of Proposition \ref{prop:homogeneous} case $(ii)$. If both $P$ and $H$ are decomposable the society will be heterogeneous (Proposition \ref{prop:homogeneous} case $(iii)$). \pu{Added. but not sure if we want to repeat this}
%\end{remark}

\begin{example}\textbf{What if $M$ is decomposable?}
Consider a situation where different groups of people have no common beliefs.
For example, people in different countries may have different sets of languages (or dialects), with no common language between the countries. Assuming people learn languages by talking to others, and that $H$ has some structure representing relatedness of the languages, we can explore the long term distribution of languages among people using our framework. %For simplicity we assume the matrices $P$ and $H$ are time homogeneous. %Similar analysis can be performed for time inhomogeneous case as well.
Notice, here the prior matrix $M$ is a block diagonal matrix. Let $P^nM=\tilde P$. If $H$ is indecomposable, the society will follow a same language distribution. However, if $H$ is decomposable but $\tilde P$ is indecomposable, then the society will stabilize to a same distribution of languages that depends on $P$, $M$ and $H$. If both $H$ and $\tilde P$ are decomposable, then the society will stabilize to a heterogenous distribution of languages.
\end{example}

%xxxxxxxxxxxxxxxxxxxxxxxxxxxxxxxxxxxxxxxxxxxxxxxxxxxxxxxxx

\subsubsection{Rate of Convergence} \label{sec:ROC}
One may ask how fast the individuals or the society attain their limiting beliefs. This provides insights to the rate of belief evolution.
%Therefore, we analyze the rate of convergence of the model. 
More precisely, what is the effect of the structure of $P$ and $H$ matrices on how fast the model converges to its stationary distribution.  We provide a lower bound on the rate of convergence  of the model that represents how quickly the sequence approaches its stationary distribution. (See definition \ref{Def:convergence rate})

\begin{comment}
\begin{prop}\cite{rosenthal1995minorization} \label{prop:EW}
An indecomposable and aperiodic Markov chain converges to its stationary distribution geometrically quickly. 
In particular, if $P$ is indecomposable and aperiodic then there exists a positive constant $C$ such that for all $i,j=1,...,N$ $$|p^n_{ij}- \pi_j|\leq C \lambda_*^n$$ where $\bm \pi=\{\pi_1,...,\pi_N\}$ is the stationary distribution. %\pu{change-ith row} %\bm 1 \pi
\end{prop}
\end{comment}

According to Proposition \ref{prop:EW}, the convergence rate of an indecomposable Markov chain is governed by the second largest eigenvalue, which is less than 1. If the chain is decomposable, it has more than one closed communicating class. We can treat each class as an indecomposable chain and find each of its rate of convergence. The slowest of those rates will be considered as the convergence rate of the decomposable chain. 
%\pw{" the SMALLEST? out of those rates will be considered as the LOWER BOUND of the convergence rate of the decomposable chain?"}\pu{if rate of convergence is large it takes more time to stabilize in to the stationary distribution, isn't it?}

% Now we explore the rate of convergence of the model. In particular, we provide a lower bound on the rate of convergence of the model to its stationary distribution. 

\begin{prop}
Suppose the network structure $P_{r \times r}$ and the concept structure $H_{s \times s}$ are indecomposable and aperiodic. Let $\lambda_P$ and $\lambda_H$ denote the second largest eigenvalues of $P$ and $H$, respectively, and
$Q_n=P^n\,M\,H^n=\left(q_{ij}\right)_{r \times s}$. Then there exists a positive constant $C_0$ such that for all $i=1,...,r$ and $j=1,...,s$
$$|q_{ij}-\pi_j|\leq C_0 (\lambda_P\, \lambda_H)^n$$
where $\bm \pi=\{\pi_1,...,\pi_N\}$ is the stationary distribution of $Q$. 
Note that $\lambda_P<1$, $\lambda_H<1$, therefore $C_0 (\lambda_P\, \lambda_H)^n \to 0$ as $n\to \infty$.
%\pw{proof or citation is missing.}
\end{prop}
%Proof follows from Proposition \ref{prop:EW} \pat{not totally clear what the logic is} 

%From Proposition \ref{prop:EW}, there exists constants $C_1$ and $ C_2$ and stationary distributions $\bm\pi^P=\{\pi_1^P,...,\pi_N^P\}$ and 
%$\bm\pi^H=\{\pi_1^H,...,\pi_N^H\}$ such that,
%$|p^n_{ij}- \pi_j^P|\leq C_1 \lambda_P^n$ and $|h^n_{ij}- \pi_j^P|\leq C_2 \lambda_H^n$. From Proposition \ref{prop:homogeneous} and letting $C_0=C_1 C_2$ the result follows. \pu{Added, but feels redundant. may be put in supplement text?}

\begin{prop}\label{prop:ROC}
Let $R_P$ and $R_H$ be the convergence rates of $P$ and $H$, respectively. Then the model converges to the stationary distribution with a rate of at least $R=\min\{R_P,R_H\}$. 
\end{prop}
%Proof follows from the above discussion.
%\pw{Probably a brief proof is need here? It is not so obvious why its max not min.}

% Suppose wlog $R=\max\{R_P,R_H\}=R_p>R_H$. Then there exists some time $t$, where $H^t$ has stabilized to its stationary distribution but $P^t$ has not.

That is the society will reach the steady state distribution only when both network and concept structures are stabilized. %\pat{we need a bit more discussion here} \pu{Once the structure with the rate of change equals to $\max\{R_P,R_H\}$ is stabilized, the structure of the belief distribution is decided, but the weight given to each belief changes till the model reaches the stationary distribution. Moreover, if $\max\{R_P,R_H\}=R_H$, then the stationary distribution is attained when $H$ stabilizes.}

%%%%%%%%%%%%%%%%%%%%%%%%%%%%%%%%%%%%%%%%%%%%%
\subsection{What if the social structure and the concept structure change over time?} %\pu{this isection is dense, write in plain language}

% \pu{change topic} \pu{this section contains the analysis of the limit of the model and convergence rate for time inhomogeneous. }

%In general, the network structure between people can change over time. %For example, the communication between some people or groups may last forever while others may stop communicating %with some people after a particular time. 
%Similarly, the concept structure may also vary at different times. Therefore, i
In this section we consider time inhomogeneous models, where network and concept structures can change over time. 
%However, we still assume that the number of people and the number of beliefs in the society remain constant. \pat{assigning zeros in the matrix can get around this} 
We provide conditions for the model convergence to homogeneous beliefs %. Moreover, we analyze the convergence problem in a \pat{what do we mean probabilistic here} probabilistic setting. Finally we provide 
convergence in expectation, and a lower bound for the rate of convergence of the model. 

\subsubsection{Convergence and limiting distribution}

For simplicity $M$ is assumed to be indecomposable in the formulation of the problem. %
For time homogeneous case, Proposition \ref{prop:homogeneous} suggests that if either $P$ or $H$ is stochastic, indecomposable and aperiodic (SIA), homogeneity of beliefs is guaranteed. 
This can be generalized to time inhomogeneous case as following:

\begin{prop}  \label{prop:rankone_inhomo}
Let $\s_{P}=\{P_i\}_{i=1}^k$ be a set of social structure matrices, and $\s_{H}=\{H_i\}_{i=1}^l$ be a set of concept structure matrices.
At each time $t$, $P_t$ and $H_t$ are chosen from $\s_{P}$ and $\s_{H}$ respectively. Then 
 $Q_n(P_t, H_t, M)=\prod_{t=0}^{n} P_t  M \prod_{t=0}^{n} H_t$ 
converges to a rank one matrix as $n \to \infty$ if and only if every possible product of matrices in $\s_{P}$ or/and $\s_{H}$ (with repetitions allowed) is SIA.
\end{prop}

%\pw{the logic for the next a few paragraphs is a bit mixed up.}
%\pw{maybe first say each in the set does not guarantee every product as in paragraph~A; then move last part of paragraph~B here, i.e. it is actually possible to check if every product is SIA. However, it is not efficient. then we try to derive a easy-to-check condition, scrambling; then after prop~\ref{prop:product_rankone}, discuss why it is good.}

Proposition \ref{prop:rankone_inhomo} provides a condition that guarantees a homogeneous belief distribution in the society in the long run. In particular, if every product in the set of network structures and the set of concept structures is SIA, the society will stabilize to a unique belief distribution.~\footnote{\cite{wolfowitz1963products} provides an algorithm to determine if every product in a given set of matrices is SIA, in a bounded number of arithmetic operations.}
However, note that each matrix in a set $\s$ (a set of stochastic matrices) being SIA does not guarantee that every product is SIA, and
% For example let $\s=\{P_1, P_2\}$ with 
% $P_1=\begin{pmatrix}
% 0 & 1 & 0\\
% 0 & 1 & 0\\
% 1 & 0 & 0
% \end{pmatrix}$ and
% $P_2= \begin{pmatrix}
% 0 & 0 & 1\\
% \frac{1}{2} & \frac{1}{2} & 0\\
% \frac{1}{3} & \frac{1}{3} & \frac{1}{3}
% \end{pmatrix}$ 
% which are SIA. However, 
% $P_1 P_2= \begin{pmatrix}
% \frac{1}{2} & \frac{1}{2} & 0\\
% \frac{1}{2} & \frac{1}{2} & 0\\
% 0 & 0 & 1
% \end{pmatrix}$ is not SIA.
%It is evident that, 
as the order of the transition matrices increases (that is, as the number of states of the Markov chain increases) it is difficult to check if every product is SIA.
Therefore we now discuss conditions on the individual matrices from $\s$ which guarantees that any product of matrices from  $\s$ is SIA.

% \begin{definition} \cite{anthonisse1977exponential} Ergodic coefficient.\\
% For any square stochastic matrix $P$, ergodic coefficient $\gamma$ is defined by $$\gamma(P)=\min_{i_1,i_2} \Sigma_{j} \min (p_{i_1j},p_{i_2j})$$ 
% \end{definition}

\begin{definition}\label{def:scrambling}
For a square stochastic matrix $P$, let $\lambda(P)= 1- \gamma$,
where $\gamma$ is the ergodic coefficient of $P$: $\gamma(P)=\min_{\{i_1,i_2\}} \Sigma_{j} \min (p_{i_1j},p_{i_2j})$.
If $\lambda(P)<1$, then $P$ is called a \textbf{scrambling} matrix.
\end{definition}

\begin{prop} \label{prop:product_rankone}
If every matrix in $\s$ is stochastic and scrambling, then any product of matrices from $\s$ converges to a rank one matrix. 
\end{prop}
%Proof \ref{proof:product_rankone} \pu{fix}
% \begin{proof}\label{proof:product_rankone}
% According to Proposition~\ref{prop:products_of_A}, any product of matrices from $\s$ converges to a rank one matrix if and only if every product of matrices in $\s$ 
% is SIA. Hence we only need to check that every product of matrices in $\s$ is SIA. 
% Since product of stochastic matrices is still stochastic, 
% Proposition~\ref{prop:scrambling implies SIA} - any stochastic scrambling matrix is SIA indicates that we only need to show that every product of matrices in $\s$ is 
% scrambling. And this holds as Proposition~\ref{lemma:scrambling} shows that if one or more matrices in a product of matrices is scrambling, so is the product. 
% Thus we are done.
% \end{proof}

%It is evident that, as the order of the transition matrices increases (that is, as the number of states of the Markov chain increases) it is difficult to check if every product is SIA. However, since any product of matrices which has a scrambling matrix as a factor is SIA, all the scrambling matrices can be disregarded at once. 
This reduces the required amount of computations as it is relatively easy to check if a matrix is scrambling or not.
Moreover, given a set $\s$, we only need to check all matrices in $\s$, rather than every possible product of matrices in $\s$.
% \pw{be more precise on the computation complexity?}

% \pw{add the following example.}
\begin{example}
Suppose there are two belief evolution systems, one with concept structure set $\s_{H} = \{H_1, H_2\}$, 
the other one with $\s'_{H} = \{H_1, H_3\}$, where
$\tiny{H_1 = \begin{pmatrix}
0.2 & 0.3 & 0.5 \\
0.6 & 0.4& 0\\
0 & 0.8 & 0.2\\
\end{pmatrix}},$
$\tiny{H_2 = \begin{pmatrix}
0.6 & 0.1 & 0.3 \\
0 & 1& 0\\
0.7 & 0 & 0.3\\
\end{pmatrix}},$
$\tiny{H_3 = \begin{pmatrix}
0 & 0.9 & 0.1 \\
0.3 & 0.2& 0.5\\
0.4 & 0.5 & 0.1\\
\end{pmatrix}}.$
Computation shows that $\gamma(H_1) = \Sigma_{j} \min (p_{2j},p_{3j}) = 0+0.4+0=0.4$, hence $\lambda(H_1)= 1- 0.4 =0.6 < 1$, $H_1$ is scrambling.
Similarly, one have $\lambda(H_2) = 1$, $\lambda(H_3) = 0.7$. Hence $H_3$ is scrambling, and $H_2$ is not.
Therefore according to Proposition~\ref{prop:product_rankone}, people in the second system must converge to the same belief.
Whether the first system converges to the same belief is further depending on its social structure set $\s_{P}$. 
%\pu{I would just show this as an example of prop 15, without relating to P or H}
\end{example}

\subsubsection{Convergence in probability setting}

%Now let's switch gear to time inhomogeneous case.
Proposition~\ref{prop:rankone_inhomo} provides a necessary and sufficient condition on when every product of stochastic matrices from $\mathcal{S} = \{A_1, \dots, A_l\}$ converges to a rank one matrix. In contrast to this absolute setting, we now consider the convergence in probability. 

% \st{Let $\mathcal{X} = \{X^1, X^2, \dots \}$ be a sequence of i.i.d. random variables with distribution $\mathcal{P}(X^t= A_i) = w_i$ where $\mathbf{w} = (w_1, \dots, w_l)\in \mathbb{R}^{l}$, $w_i>0$ and $\sum_{i=1}^{i=l} w_i =1$. Then each sequence $B^1, B^2, \dots $ formed by matrices in $\s$ can be viewed as a realization of $\mathcal{X}$. }

% to randomly form a sequence $B^1, B^2, \dots $ from $\mathcal{S}$, we assume that each time a matrix $B^k$ is sampled independently and 
% identically from $\mathcal{S}$ with a given distribution $\mathbf{w} = (w_1, \dots, w_l)$ where $\sum_{i=1}^{i=l} w_i =1$, i.e. $\mathcal{P}(B^k= A_i) = w_i$. Without loss, we further assume $w_i > 0$ for all $i$, otherwise, one can simply remove $A_i$ from $\mathcal{S}$. Then the following holds,

\begin{prop}\label{prop: prob_one_converge} %\footnote{See all proofs in Supplemental Material~\ref{apd:inhomo}.}
Given a set of stochastic matrices $\mathcal{S} = \{A_1, \dots, A_l\}$ and a positive vector 
$\mathbf{w} = (w_1, \dots, w_l) \in \mathbb{R}^l$ with $w_i>0$ and $\sum_{i=1}^{i=l} w_i =1$, 
a product of matrices that are i.i.d. sampled from $\s$ according to $\mathbf{w}$ converges to a rank one matrix with probability one 
if and only if there exists a finite product 
% $B = \prod_{i=1}^{N} B^i$ of matrices from $\mathcal{S}$ such that $B$ is scrambling. 
$B = \prod_{i=1}^{N} B_i$ of matrices, where $B_i$ is from $\mathcal{S}$ such that $B$ is scrambling. 
\end{prop}

\begin{remark}
We may replace `scrambling' in Proposition~\ref{prop: prob_one_converge} by `SIA' as sufficiently large powers of an SIA matrix are scrambling and 
any product that has a scrambling matrix as a factor is SIA. %\pat{discuss}
\end{remark}

%\pw{the following of this subsection can be moved to appendix?}
%Section \ref{sec:GT} presents a graph theoretic interpretation of Markov chains and associated definitions. 

%In particular, 
% \pw{Move the next part to supp, then swap to the following text, see if it is ok.}

Although it is easy to check if a matrix is scrambling, to make sure whether a scrambling product $B$ exist in Proposition~\ref{prop: prob_one_converge} could still be challenging.
We now introduce an equivalent condition in form of graphs, which is straightforward to verify.

Associated with the finite state Markov chain of a transition matrix $A$, there is a directed graph $G_{A}$ 
\footnote{Refer to Supplemental Material~\ref{sec:GT} for a detailed graph theoretic interpretation of Markov chains.}.
For instance, let
$\small{A=
\begin{blockarray}{cccc}
&v_1 & v_2  & v_3 \\
\begin{block}{c (ccc)}
 v_1 & 0 & 0.7&0.3 \\
 v_2 & 0 & 1 &0 \\
 v_3 & 0& 0&1\\
\end{block}
\end{blockarray}}$, then the corresponding graph is 
\begin{tikzpicture}[->]
    \node[ellipse,draw] (C) at (0,0.5) {$v_1$};
    \node[ellipse,draw] (N) at (-1,-0.5) {$v_2$};
    \node[ellipse,draw] (Y) at (1,-0.5) {$v_3$};
    \path (C) edge              node[above]{} (N);
    \path (C) edge              node[above,right] {}(Y);
    \path (N) edge [loop left] node {} (N);
    \path (Y) edge [loop right] node {} (Y);
\end{tikzpicture}.

\noindent Similarly, associated with a set of transition matrices $\mathcal{S} = \{A_1, \dots, A_l\}$ with a fixed collection of states,
we may define a directed graph $G_{\mathcal{S}}$, where $G_{\mathcal{S}}$ has the same vertex set 
as any $G_{A_k}$, and the edge set contains $(i,j)$ if there exists a $k \in \{1, \dots, l\}$ such that $G_{A_k}$ contains $(i,j)$.
For instance, let $\mathcal{S} = \{A_1, A_2, A_3\}$, where $A_1 = A$ as above, 
$\small{A_2=
\begin{blockarray}{cccc}
&v_1 & v_2  & v_3 \\
\begin{block}{c (ccc)}
 v_1 & 1 & 0 & 0 \\
 v_2 & 0 & 0.6 &0.4 \\
 v_3 & 0& 0&1\\
\end{block}
\end{blockarray}}$ and 
$\small{A_2=
\begin{blockarray}{cccc}
&v_1 & v_2  & v_3 \\
\begin{block}{c (ccc)}
 v_1 & 0.3 & 0.2 & 0.5 \\
 v_2 & 0 & 1 &0 \\
 v_3 & 0& 0.8& 0.2\\
\end{block}
\end{blockarray}}$, then the corresponding graph $G_{\mathcal{S}}$ is:
\begin{tikzpicture}[->]
    \node[ellipse,draw] (C) at (0,0.5) {$v_1$};
    \node[ellipse,draw] (N) at (-1,-0.5) {$v_2$};
    \node[ellipse,draw] (Y) at (1,-0.5) {$v_3$};
    \path (C) edge              node[above]{} (N);
    \path (C) edge              node[above,right] {}(Y);
    \path (N) edge              node[above]{} (Y);
    \path (Y) edge              node[above]{} (N);
    \path (C) edge [loop above] node {} (C);
    \path (N) edge [loop left] node {} (N);
    \path (Y) edge [loop right] node {} (Y);
\end{tikzpicture}.
A set of vertices are said to be \textit{strongly connected} if there exists a directed path between any pairs of vertices in the set.
Each $G_{\s}$ further induces a condensed graph $\widehat{G}_{\s}$ by combining vertices in each strongly connected set into a `super-vertex'. 
In our example $v_2$ is strongly connected to $v_3$. Hence we have $\widehat{G}_{\s}$ as: 
\begin{tikzpicture}[->]
    \node[ellipse,draw] (C) at (0,0.5) {$v_1$};
    \node[ellipse,draw] (N) at (0,-0.5) {$v_{23}$};
    \path (C) edge              node[above]{} (N);
    \path (C) edge [loop above] node {} (C);
    \path (N) edge [loop left] node {} (N);
\end{tikzpicture}
A state is defined to be \textit{recurrent} if it is contained in a leaf of $\widehat{G}_{\s}$, otherwise the state is \textit{transient}.
Thus in our example, $v_2, v_3$ are recurrent, and $v_1$ is transient.

Notice that if $\widehat{G}_{\s}$ is connected and has one leaf, this is equivalent to a finite product of matrices from $\mathcal{S}$ which are scrambling\footnote{See proof in Supplemental Material~\ref{apd:inhomo},}, Hence, as a consequence of Proposition~\ref{prop: prob_one_converge}, we have,

\begin{cor}\label{cor: one_leaf}
Given a time-inhomogeneous Markov chain with $\s$ and $\mathbf{w}$, a product of transition matrices that are i.i.d. sampled from $\s$ according to $\mathbf{w}$ converges to a rank one matrix with probability one if and only if $\widehat{G}_{\s}$ is connected and has one leaf. %\pu{please define $\widehat{G}_{\s}$ here}
\end{cor}

% \pw{move the following remark to supp.}
% \begin{remark}\label{rmk: diverge}
% \st{Now assume $\widehat{G}_{\s}$ is connected and has more than one leaf.
% If there exist two classes of recurrent states that are accessible from the same class of transient state, i.e. 
% two leaf vertices in $\widehat{G}_{\s}$ have a common ancestor, then an i.i.d. sampled product diverges (not converge) with probability $1$ for many choices of $\s$.
% In this case, recurrent states are eventually stabilized, but transient states are mixtures of recurrent states where the mixture weights varies as different 
% transition matrices are sampled.}
% % $\widehat{G}_{\s}$ has more than one leaf nodes implies that every matrix in $\s$ is decomposable. In particular, they can be arranged to share 
% %a common block upper triangular structure with two diagonal blocks. 
% \end{remark}

The limit of the product of sampled transition matrices may not exist when there are more than one leaf of $\widehat{G}_{\s}$.
Hence instead we now consider the expectation of such limit. % and the following holds.

\begin{prop}\label{prop: expectation}
Given a time-inhomogeneous Markov chain with $\s$ and $\mathbf{w}$, the expectation of the product of sampled transition matrices 
is equal to the limiting product of the expectation of the transition matrix, i.e. 
$ \mathbb{E} (\lim_{N\to \infty} \prod_{t=1}^{N}X^{t}) =\lim_{N\to \infty} \prod_{t=1}^{N} \mathbb{E} (X^{t})$, 
where $\avg (X^t) = \sum_{k=1}^{k=l} w_k \cdot A_k $. %\triangleq \bar{A}
\end{prop}
%\pu{please add an example}
%\noindent \textbf{Convergence of the model:} 

Based on the above analysis, we can now investigate long-term behavior when both network and concept structures are sampled from a collection of matrices 
$\s_{P} = \{P_1, \dots, P_k\}$ and $\s_{H} = \{H_1, \dots, H_l\}$ respectively. 
Let the condensed graphs corresponding to $\s_{P}$ and $\s_{H}$ be $\widehat{G}_{P}$ and $\widehat{G}_{H}$. 
According to Proposition~\ref{prop: prob_one_converge} and Corollary~\ref{cor: one_leaf}, analogous to Proposition~\ref{prop:homogeneous} for the time homogeneous case, 
the following holds.

When $\widehat{G}_{H}$ has only one leaf, or equivalent there is a finite product from $\s_{H}$ is scrambling or SIA, then with probability one 
everyone in the network converges to the same posterior distribution over the hypothesis set $\mathcal{H}$. In particular, the posterior distribution is 
supported only on the recurrent hypotheses, i.e. hypotheses in the leaf.
When $\widehat{G}_{H}$ has more than one leaf, or equivalently there is a common indecomposable structure for every matrix in $\s_{H}$, 
but $\widehat{G}_{P}$ has only one leaf, then everyone in the network still converges to the same posterior distribution over $\mathcal{H}$ with probability one. 
Moreover, the shared posterior distribution is a mixture of isolated posterior distribution of recurrent people (people in the leaf vertices).
Thus, the shared posterior distribution is completely determined by priors of recurrent people and their belief's corresponding confusion parameters.
    
When both $\widehat{G}_{H}$ and $\widehat{G}_{P}$ have more than one leaf, people in different recurrent classes (people in different leaf vertices) converge to different posterior distributions. In general, posteriors of people in transient states is a mixture of posteriors for recurrent classes where the mixture weights differ over time (no limit exists).
%as in Remark~\ref{rmk: diverge}. 

In all cases, Proposition~\ref{prop: expectation} suggests that the expectation of people's posterior is:
$\lim_{n \to \infty} (\avg P)^n \cdot M \cdot(\avg H)^n = (\avg P)^{\infty} \cdot M \cdot(\avg H)^{\infty}$.

% \pu{implications}

\subsubsection{Rate of convergence} 
Next, we explore the rate of convergence of inhomogeneous Markov chains. We then discuss how to obtain the rate of convergence of the model when both $P_n$ and $H_n$ are time inhomogeneous.

\begin{prop} \cite{anthonisse1977exponential} 
Suppose that any product of matrices from $\s$ converges to a rank one matrix. Then there exist an integer $\nu \geq 1$,
for any sequence $\{A_1,...,A_n \}$, $n \geq 1$ of matrices from $\s$, such that for all $i, j = 1,..., N$,
$$|(A_1 ... A_n)_{ij}-\pi_j | \leq (1-\gamma)^{[n/\nu]} $$ for all $n \geq 1$, 
where $\bm \pi=(\pi_1, \dots, \pi_N)$ is the stationary distribution, $\gamma= \min \{ \gamma(A_1...A_\nu)|A_i \in S, 1\leq i \leq \nu \}$, and $[x]$ is the largest integer less than or equal to $x$.

%$\bm \pi=\{\pi_j, 1\leq j \leq N\}$ is the stationary distribution, $\gamma= \min \{ \gamma(A_1...A_\nu)|A_i \in S, 1\leq i \leq \nu \}$, and $[x]$ is the largest integer less than or equal to $x$. 
\end{prop}

In other words, the above proposition provides an upper bound for the rate at which the network structure (or concept structure) stabilizes, for any SIA product of matrices in $\s_{P}$(or $\s_H$). Note that this depends on the ergodic coefficients of the matrices. Integer $\nu$ can always be taken
less than or equal to $\nu^* = \frac{1}{2}(3^N - 2^{N+1} + 1)$ \cite{anthonisse1977exponential}. 
Rate of convergence of an indecomposable Markov chain with transition matrices from $\s$ is upper bounded by $(1-\gamma)^{[n/\nu^*]}$. If the transition matrices are decomposable, we can perform similar analysis as discussed in section \ref{sec:ROC} by considering the convergence rate of each communicating class.

Now, we look at the convergence rate of the model when the network structure and the concept structure change over time. That is $P_n$ and $H_n$ are inhomogeneous. We assume that $P_n \in \s_P$ and $H_n \in \s_H$ where $\s_P$ is a finite set of social structure matrices and $\s_H$ is a finite set of concept structures. In other words, at each time step, people's network structure takes the form of a stochastic matrix from the finite set $\s_P$. Similarly, for $\s_H$. 

\begin{prop}
Let $R_P$ and $R_H$ be the convergence rates of $P_n$ and $H_n$, respectively. Then the model converges to its stationary distribution with a rate of at least $R=\min\{R_P,R_H\}$.
\end{prop}
Proof follows from an argument similar to Proposition \ref{prop:ROC}. The society will reach the steady state distribution only when both network and concept structures are stabilized.

%\pagebreak
 \section{Belief evolution over dynamic, homophily-based networks} \label{sec:KL}
%\pu{we discuss our approach to create P and H given M, using KL divergence, discuss the results, give an example, theorem: lower bound of the number of groups}

Results in the previous section assume that network structures are either static or change at random. However, network structures in society, especially in terms of who we communicate with, are affected by our beliefs \cite{liu2018competition, mcpherson2001birds, murase2019structural}. For example, people may be more likely to talk with people whose beliefs are more similar to their own, either because of consistency of beliefs in a geographic region \cite{cepic2020social, khanam2020homophily}, or through active selection of partners. 
%\pat{define affinity}\pu{a liking for someone or something, especially because of shared characteristics/ a similarity of characteristics suggesting a relationship, especially a resemblance in structure between animals, plants, or languages.}
Because beliefs change based on who one talks with, networks that are based on homophily may be dynamic. 
In this section, we analyze belief evolution for societies whose structures are governed by homophily. 

% \hl{We should talk through the structure of the rest of this section.} 
% \pat{proposal: mirror structure of section 3? We have less detail, but that would help the reader.}
% \pat{proposal: focus on explaining the formalization of affinity and the proof. move simulations to the appendix / de-emphasize}

%We provide a computational framework to illustrate the belief evolution modeled in Section \ref{Sec:Formulation} through simulations. 
Given people's initial priors $M$ we create the network structure of people and cognitive structure of concepts based on belief similarity. 
Specifically, we construct the homophily network structure by linking people whose beliefs are sufficiently similar. 
Let $M_n$ be the matrix representing beliefs of individuals at time~$n$. 
Further $S: \mathbb{R}^n \times \mathbb{R}^n \to \mathbb{R}^{*}$ is a function that measures divergence between to vectors,
where $S(\mathbf{v}, \mathbf{u}) = 0$ indicates $\mathbf{v} = \mathbf{u}$.
Then for a given similarity threshold $\epsilon_p>0$, individuals $\alpha_i$ and $\alpha_j$ are linked, i.e. $p_{ij} > 0$, 
if $S(\mathbf{p}_i, \mathbf{p}_j) < \epsilon_p$ where $\mathbf{p}_i, \mathbf{p}_j$ are the row vectors in $M_n$ corresponding to $\alpha_i$ and $\alpha_j$; otherwise $p_{ij} = 0$. %Moreover, values of $p_{ij} >0$  

Similarly, we construct the homophily concept structure by linking concepts that are held to similar degrees across people.
In particular, let $\widehat{M}_n$ be the column normalization of $M_n$. Then for a given similarity threshold $\epsilon_h>0$, 
concepts $\beta_k$ and $\beta_l$ are linked, i.e. $h_{kl}>0$ if $S(\mathbf{h}_k,\mathbf{h}_l) < \epsilon_h$ where $\mathbf{h}_k, \mathbf{h}_l$ 
are the column vectors of $\widehat{M}_n$ corresponding to $\beta_k$ and $\beta_l$; otherwise, $h_{kl} = 0$. 
In this section, we measure similarity of beliefs between pairs of people, and of degrees between pairs of concepts via Kullback-Leibler (KL) divergence. 
In addition to being a natural measure of divergence between beliefs, KL divergence is asymmetric, which means that our network and 
concept structures are not restricted to be symmetric.

\begin{definition} \textbf{[KL divergence]}
Let $A=(a_{ik})$ be a row stochastic matrix.
Define KL divergence between two discrete probability distributions, $p$ and $q$, in $\mathbb{R}^K$, 
$$\text{KL}(p,q)=\sum_{k \in \mathcal K} p_k \log\left(\frac{p_k}{q_k}\right).$$
\end{definition}

% Once we determine if there is a link or not between people (or beliefs) in the network (or concept) structure based on the similarity function, we calculate the strength of the links as a relative weighted average of the similarities. In particular, let $S_{ij}={S(\mathbf{p}_i,\mathbf{p}_j)}$ 
% We use the Soft max function to obtain these probabilistic links which characterize the weights of each link such that they are proportional to the relative scale of the similarity.

If two individuals or concepts are sufficiently similar, they will be linked. Next we calculate the strength of the links as a relative divergence. In particular, the strength of the link is related to their divergence relative to other linked individuals or concepts by the softmax function.

\begin{definition}\textbf{[Softmax function]} For a given vector, $\mathbf{a} = (a_1, \dots, a_N) \in \mathbb{R}^N$ and a parameter $\beta$, the softmax function $\sigma$ of $\mathbf{a}$ is, 
$\sigma(\mathbf{a})=\frac{e^{-\beta \mathbf{a}}}{\sum_{i=1}^N e^{-\beta a_i}}$.
\end{definition}

We define the weights of the links between individuals as follows: Let $S_{ij}={S(\mathbf{p}_i,\mathbf{p}_j)}$ be the similarity of beliefs between individuals $\alpha_i$ and $\alpha_j$ and $\mathbf{S}_i=\{S_{i1},S_{i2},...S_{ir}\}$.
We define $\mathbf{w}_i=\sigma(\mathbf{S}_i)$, where $\mathbf{w}_i$ be the vector with weights of the probabilistic links from $\alpha_i$ to $\alpha_j$, $\forall j=\{1,2,..,r\}$. Similarly we define the weights of the links for concept structure using Softmax function.

% \begin{alg} \label{alg1}

% \noindent Input: $\epsilon_p=$ threshold for $P$, $\epsilon_h=$ threshold for $H$, $N=$max number of steps, $r=$ number of people, $s=$ number of concepts.

% \noindent Output: $P_t, Q_t$

% \noindent Initialize $M_{r\times s}$ using a Dirichlet distribution, $P_0=I_{r \times r}, H_0=I_{s\times s}, Q_0=P_0 M H_0$

% \noindent For t from 1 to N calculate:

% $KL^p_{mat}=(kl^p_{ij})$ such that $kl^p_{ij}=KL(\alpha_i,\alpha_j)=\sum_{\beta_k \in \mathcal H} m_{ik} \log\left(\frac{m_{ik}}{m_{jk}}\right),$
    
% % $SM_{mat}=(SM_{ij})$ such that $SM_{ij}= \sigma(kl_{ij})$
    
% $P_t=(p^t_{ij})$ where
%  $p^t_{ij}=
%   \begin{cases}
%       \sigma(kl^p_{ij}), & \text{if}\,\,\,  kl^p_{ij}<\epsilon_p \\
%       0, & \text{otherwise}
%     \end{cases}$

% $P_t=$row normalize $P_t$

% $\Tilde{M}$=column normalize $M$

% $KL^h_{mat}=(kl^h_{ij})$ such that $kl^h_{ij}=KL(\beta_i,\beta_j)=\sum_{\alpha_k \in \mathcal P} \tilde{m}_{ki} \log\left(\frac{\tilde{m}_{k_i}}{\tilde{m}_{kj}}\right),$

% $H_t=(h^t_{ij})$ where
%  $h^t_{ij}=
%   \begin{cases}
%     %   \frac{1}{\sigma(kl^h_{ij})}, & \text{if}\,\,\,  kl^h_{ij}<\epsilon_h \\
%      \sigma(kl^h_{ij}), & \text{if}\,\,\,  kl^h_{ij}<\epsilon_h \\
%       0, & \text{otherwise}
%     \end{cases}$
% % \pw{Again, it feels that the softmax for $H_t$ does not need a reciprocal. }

% $H_t=$row normalize $H_t$

% $Q_t=P_tQ_{t-1}H_t$

% $M=Q_t$

% \end{alg}
We now introduce the homophily-based model, which at each time step adapts its structure on $P_{n+1}$ and $H_{n+1}$ based on $M_n$,
%That is, we assume that $\forall t, P_t=P$ and $H_t=H$; $P$ and $H$ are constant matrices. Then the operator $Q$ simplifies to 
%$$Q_{n}(P,H,M)=P^n\,M \,H^n.$$
% For a square matrix $A$, $A^n$ denotes the multiplication of $A$ for $n$ times. 
$$ Q_{n+1}(P_t, H_t, M)=\prod_{t=0}^{t=n+1} P_t \cdot M \cdot\prod_{t=0}^{t=n+1} H_{t} = P_{n+1} M_{n} H_{n+1}$$
where $M_n$ is the matrix representing beliefs of individuals, $j \in \{1,2, \dots, r\}$, in concepts, $k\in \{1,2, \dots, s\}$, at time $n$
%\pat{check total number of people total number of concepts and be consistent} 
and $P_{n+1}$ is the network structure matrix and $H_{n+1}$ is the concept structure derived from $M_n$ as described above. 
One question we may ask is whether the dynamic nature of the homophily structures yield interesting changes in the asymptotic structure of 
the society. We have seen from previous results that as long as one of the network or concept structures is indecomposable, the long run behavior is that everyone converges to a single group with the same beliefs. We now prove a lower bound on the number of groups of beliefs for homophily-based dynamic structures, which shows the same does not hold. 
%The following theorem provides a lower bound for the number of groups that will remain in the society in the long run.

%\pw{please see the following for the thm of lower bound of the number of groups.}
%\pw{i assume `the number of groups' is defined somewhere before}

\begin{definition}
Let $\mathcal{V} = \{\mathbf{v_1},\dots, \mathbf{v_k}\} \subset \mathbb{R}^{n}$ be a set of vectors. 
Given $\e >0$, a \textbf{$\e_{p}$-KL cluster} over $\mathcal{V}$ is defined to be a subset $V\subset \mathcal{V}$ such that:
for any $\mathbf{v}_{i} \in V$, $\text{KL}(\text{Conv}(V_{-i}),\mathbf{v}_{i})< \e$ holds,
and for any $\mathbf{v}_{j} \notin V$, $ \text{KL}(\text{Conv}(V),\mathbf{v}_{j} )\geq \e$ holds, 
where $V_{-i}$ represents omitting $\mathbf{v}_i$ in $V$, and $\text{Conv}(\cdot)$ represents the convex hull. 
\end{definition}

\begin{thm}\label{thm: lowerbound}
Let $M$ be the initial belief matrix and $\e_p$ be the threshold of network structure.
Assume the concept structure $H$ is the identity.
Then the number of groups in network structure (number of communicating classes in $P_t$) is bounded below by the number \textbf{$\e_{p}$-KL clusters} 
over row vectors of $M$. Similarly, assume $P$ is identity, 
then the number of groups in concept structure (number of communicating classes in $H_t$) is bounded below by the number \textbf{$\e_{h}$-KL clusters} 
over column vectors of $M$.
\end{thm}

We first describe an algorithm to construct \textbf{$\e_p$-KL clusters}, the proof then follows along.
\begin{itemize}
    \item \textbf{Step 0} Each row of $M$ (representing belief of a person) can be realized as a point $p^0_i \in \mathbb{R}^s$. 
    View each point as a vertex (representing a person) to obtain $G_{P_0}$;
    \item \textbf{Step 1} For a pair of vertices $\alpha_i$ and $\alpha_j$, add an edge $e_{ij}$ if $\text{KL}(p^0_i,p^0_j) < \e_p$ to obtain $G_{P_1}$. \textit{Note}: Let $V$ be the vertex set of a connected component of $G_{P_1}$, and $\alpha_l \in V$ be a person belongs to this group. 
    Then since $\alpha_l$'s belief  will be updated as a linear combination of concepts in $V$, 
    the point $p^0_l$ in $\mathbb{R}^s$ representing $\alpha_l$ can only move to a new point $p^1_l$ in the convex hull $\text{Conv}(V)$ of $V$.

    \item \textbf{Step 2} For each pair of connected components $V_1$ and $V_2$ of $G_{P_k}$, 
    %and each person $\alpha_k \notin V$ with the point representing $\alpha_k$'s belief in $B_{\e_p}^{KL} (\text{Conv}(V)))$, 
    if $B_{\e_p}^{KL} (\text{Conv}(V_1))\cap\text{Conv}(V_2) \neq \emptyset$ or $B_{\e_p}^{KL} (\text{Conv}(V_2)))\cap\text{Conv}(V_1) \neq \emptyset$, thenadd an edge from a person $\alpha_i$ in $V_1$ to  $\alpha_j$ in $V_2$ to obtain $G_{P_{k+1}}$.
    Here $B_{\e_p}^{KL} (\text{Conv}(V)) \subset \mathbb{R}^s$ is the set that contains all points within $\e_p$ 
    close of $\text{Conv}(V))$ measured by KL-divergence, i.e.
    for any $p \in B_{\e_p}^{KL} (\text{Conv}(V)) \subset \mathbb{R}^s$, there exists a $q \in \text{Conv}(V)$ such that $\mathbf{KL}(q,p) < \e_p$.
    
    \item Repeat \textbf{step 2} until connected components of $G_{P_k}$ stabilize, denote the converged graph by $G_{P}$.

\end{itemize}

It is clear from the above construction that each vertex set of a connected component of $G_{P}$ forms a $\textbf{$\e_{p}$-KL cluster}$ over rows of $M$.
On the other hand, communication can only happen between people within the same connected component of $G_{P}$ for any choice of $t$. Indeed, if $\alpha_i$ and 
$\alpha_j$ communicate time $t^*$, i.e. $\text{KL}(p^{t^*}_i, p^{t^*}_j) < \epsilon$ or $\text{KL}(p^{t^*}_j, p^{t^*}_i) < \epsilon$, the connected components containing $\alpha_i$ and $\alpha_j$ will be connected for any $t>t^*$.
Hence, the number of groups in network structure (number of communicating classes in $P_t$) is bounded below by the number \textbf{$\e_{p}$-KL clusters} 
over row vectors of $M$. Thus, Theorem~\ref{thm: lowerbound} holds. Even though we use KL divergence as a natural choice, above results hold for any divergence.

Even though $P_t$ and $H_t$ change over time in the homophily model, numerical simulations show that the model converges to its stationary distribution after few steps. %The behavior at this timepoint is as described in Section \ref{sec:limit time homo}. 
%\pw{to be really careful, I do not think we have a convergence proof when both P and H change at the same time}
Using this framework we can illustrate various behavior including isolated individuals, emergence of subgroups, and evolving into a society with a homogeneous belief system.  Example \ref{EX:sim1} below shows how society evolves into sub communities with people in the same group having the same belief distribution, and that isolated individuals are also possible.

\begin{example} \label{EX:sim1}
Consider a network with  $ \epsilon_p=0.3, \epsilon_h=0.25$
% $r=5, s=4$,
and 
$
\small
M={ \begin{pmatrix}
0.263 & 0.472 & 0.084 & 0.181\\
0.06 & 0.103 & 0.683 & 0.154\\
0.69 & 0.176 & 0.043 & 0.091\\
0.029 & 0.136 & 0.479 & 0.357\\
0.252& 0.463& 0.08 & 0.204
\end{pmatrix}}$.
Using the homophily-based framework developed above we get, 
% when $t=1$ :
$P_1= \small{
\begin{pmatrix}
0.5  & 0&    0 &   0  &  0.5  \\
 0 &   0.532 & 0  &  0.468 & 0  \\
 0  &  0  &  1  &  0 &   0   \\
 0 &   0.464 & 0 &   0.536  &0   \\
 0.5  & 0   & 0  &  0 &   0.5  
 \end{pmatrix}}, H_1= I_{4\times 4}.$
For all $t\geq 2$: 
$P_t=\small{
\begin{pmatrix}
0.5 & 0 &  0 & 0 & 0.5\\
 0 &  0.5 & 0 &  0.5&  0 \\
0 &  0 &  1 &  0&   0 \\
0 &  0.5 &  0 &  0.5 &  0 \\
0.5 & 0 &  0  & 0 &  0.5
 \end{pmatrix}},$ and for all $t\geq 5$: $
 H_t=\small{
\begin{pmatrix}
1 &  0 &  0&   0 \\
0 &  1 &  0 &  0 \\
0 &  0 &  0.5 & 0.5\\
0 &  0 &  0.5&  0.5
  \end{pmatrix}}$.
and 
% $Q_t=\small{
% \begin{pmatrix}
% 0.258 &  0.468&  0.046 &  0.228\\
% 0.045&  0.119 & 0.141 & 0.695\\
%  0.69 &  0.176&  0.023 & 0.112\\
%  0.045&  0.119&  0.141&  0.695\\
%  0.258&  0.468 & 0.046 & 0.228
%   \end{pmatrix}}$
$Q_t=\small{
\begin{pmatrix}
0.258 & 0.468 & 0.095 & 0.179\\
0.044 & 0.12&  0.29&  0.546\\
0.69&  0.176& 0.046& 0.088\\
0.044& 0.12&  0.29&  0.546\\
0.258& 0.468& 0.095& 0.179
   \end{pmatrix}}$ 
   %refer Z_G_barabasy-us-4_KLpnew_KLhnew-NoInverse.ipynb
%   Q [[0.258 0.468 0.095 0.179]
%  [0.044 0.12  0.29  0.546]
%  [0.69  0.176 0.046 0.088]
%  [0.044 0.12  0.29  0.546]
%  [0.258 0.468 0.095 0.179]]

All the matrices are rounded up to 3 decimal places. Notice that when $t=1$, $H$ does not have any links. At $t=5$, $P$ and $H$ have stabilized to stationary distribution $lim_{n\to \infty}P_n$ and $lim_{n\to \infty}H_n$, respectively. $Q_4$ shows the limiting belief distribution of the society, $lim_{n\to \infty}Q_n$. Here, three subgroups has emerged: groups ${\alpha_1,\alpha_5}$ and ${\alpha_2,\alpha_3}$ and the isolated individual $\alpha_3$. We can see that people in each subgroup has their unique belief distribution.

\end{example}

It is intuitive that when $\epsilon_p$ or $\epsilon_h$ is increased enough while the number of people and number of belief are fixed, the society will display a homogeneous belief distribution in the long run.  Example \ref{EX:sim2} illustrates this scenario.
\begin{example} \label{EX:sim2}
Next we consider the same $M$ and $\epsilon_p$ as in \ref{EX:sim1} and let $ \epsilon_h=0.4$. Then the society stabilizes to a unique belief distribution. In particular, for all $t\geq 4$:
% $Q_t=\small{
% \begin{pmatrix}
% 0.282 & 0.199 & 0.146 & 0.372\\
%  0.282 & 0.199 &  0.146 & 0.372\\
%  0.282 & 0.199 &  0.146 & 0.372\\
%  0.282 & 0.199 &  0.146 & 0.372\\
%  0.282 & 0.199 & 0.146 & 0.372
%   \end{pmatrix}}$.
$Q_t=\small{
\begin{pmatrix}
0.295& 0.186& 0.258& 0.262\\
0.295& 0.186& 0.258& 0.262\\
0.295& 0.186& 0.258& 0.262\\
0.295& 0.186& 0.258& 0.262\\
0.295& 0.186& 0.258& 0.262
   \end{pmatrix}}$.
   %refer Z_G_barabasy-us-4_KLpnew_KLhnew-NoInverse.ipynb
   
% 6 Q [[0.295 0.186 0.258 0.262]
%  [0.295 0.186 0.258 0.262]
%  [0.295 0.186 0.258 0.262]
%  [0.295 0.186 0.258 0.262]
%  [0.295 0.186 0.258 0.262]]
\end{example}

Further examples below illustrates the evolution of homophily networks over time. Namely, Examples \ref{Ex:KL1}, \ref{Ex:KL2} and \ref{Ex:KL3} show how each person's beliefs %and the KL region 
evolves with time, for different initial $M$ matrices and threshold values. In each example we consider three concepts $h_1,h_2,h_3$. The concept space is represented by an equilateral triangle with vertices $h_1\equiv(1,0,0), h_2\equiv(0,1,0), h_3\equiv(0,0,1)$.
Each person's belief at time $t$ 
%is a linear combination of $h_1,h_2$ and $h_3$ and 
is denoted by a point inside the triangle. 
All the points $\{\mathbf{x}\in \mathbb{R}^3: \text{KL}(\mathbf{p}_i,\mathbf{x}) < \epsilon_p\}$ for each person $\alpha_i$ (referred as KL regions), are represented by the coloured regions, at every time step until the model converges. At any time step, if a person $\alpha_j$ is in the KL region of another person $\alpha_i$, then $\alpha_i$ creates a communication link with $\alpha_j$, represented by a line connecting the two corresponding points. (These are unidirectional links, however we show them by a line). Observe that these regions evolve with time.
\ref{Ex:KL1} illustrates that the
creation as well as destroying of network links are possible. Moreover, the changes in threshold parameters changes the limiting behavior. In \ref{Ex:KL2} the society converge to two groups. However in \ref{Ex:KL2} where we increase $\epsilon_h$ while keeping everything else the same, society converge to  one group. 

\begin{example} \label{Ex:KL1} 
%refer Z_G_barabasy-us-4_KLpnew_KLhnew-demo-NoInverse.ipynb
We consider five people $i=1,...,5$ and three concepts $h_1,h_2,h_3$. Initial $M$ is
$\small{ \begin{pmatrix}
0.348 & 0.039 & 0.6132\\
0.321 & 0.609& 0.07\\
0.884& 0.083& 0.033 \\
0.082 & 0.185& 0.733\\
0.372 & 0.281 & 0.347 
\end{pmatrix}}$, $\epsilon_p=0.3$ and $\epsilon_h=0.2$.
At each time step, each individual's KL region changes, leading to creating new communication links or destroying existing ones. We observe that at t=5, the society stabilizes and persons $i=1$ and $2$ become isolated while the others converge to one subgroup (Figure~\ref{F1}). 

% eps 0.3
\begin{figure}[H] 
\begin{minipage}{.28\textwidth}
  \includegraphics[scale=0.2]{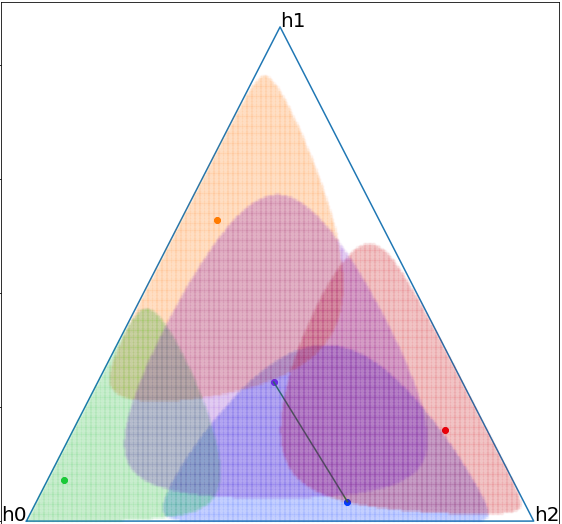}
   \caption*{t=1}
 \end{minipage}
 \begin{minipage}{.25\textwidth}
  \includegraphics[scale=0.2]{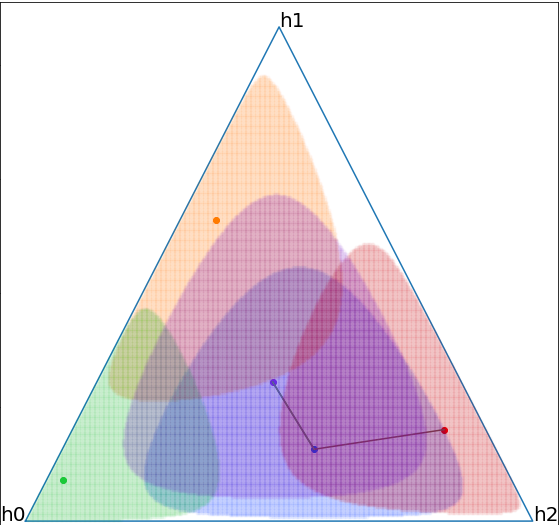}
\caption*{t=2}
 \end{minipage}
 \begin{minipage}{.28\textwidth}
  \includegraphics[scale=0.2]{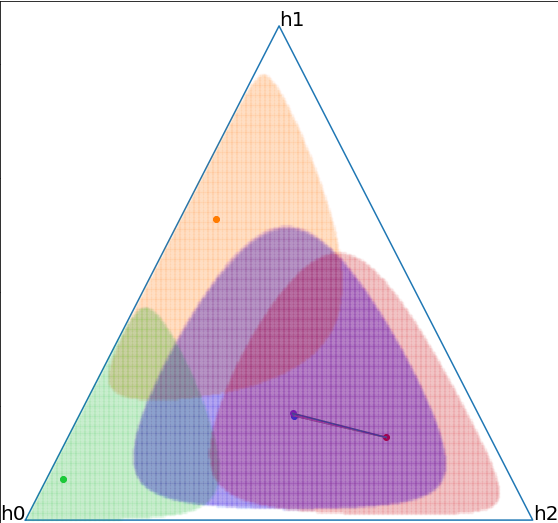}
\caption*{t=3}
 \end{minipage}\\
 \begin{minipage}{.28\textwidth}
  \includegraphics[scale=0.2]{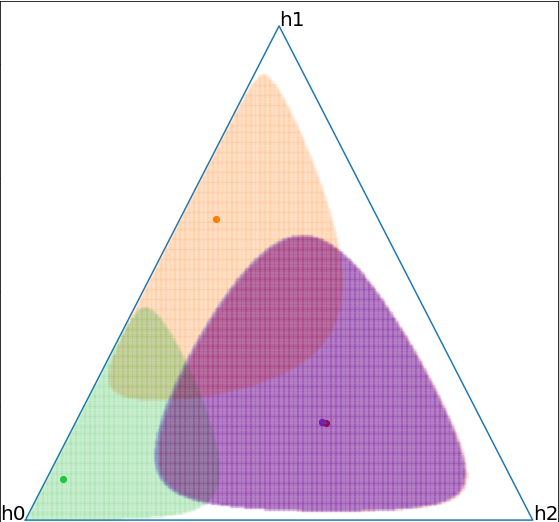}
\caption*{t=4}
 \end{minipage}
 \begin{minipage}{.28\textwidth}
  \includegraphics[scale=0.2]{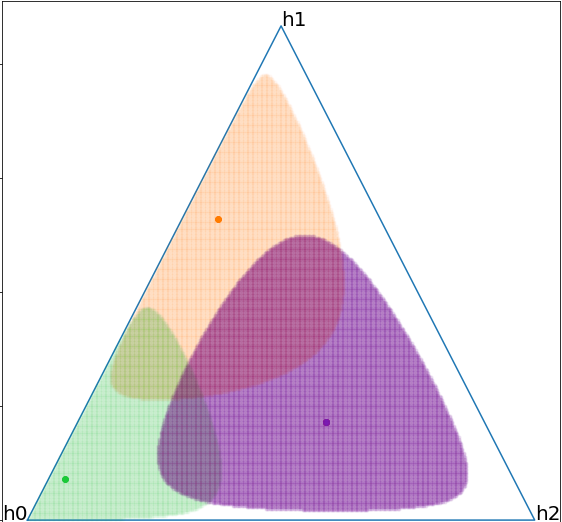}
  \caption*{t=5}
 \end{minipage} 
 \begin{minipage}{.5\textwidth}
 \hspace{15mm}
 \includegraphics[]{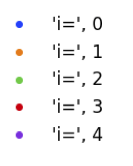}
%   \caption{t=5}
 \end{minipage} 
\caption{Each colored point inside a triangle represents the belief of a person $i$ at time $t$. The corresponding colored region represents the KL region of the person $i$. A line that connects two points denotes that the corresponding two people are communicating. } \label{F1}
\end{figure}
\end{example}
\begin{example} \label{Ex:KL2}In this example, we consider four people $i=1,...,4$ and three concepts. Initial $M$ is 
$\small{ \begin{pmatrix} 
0.489& 0.104& 0.407\\
0.033& 0.712& 0.255\\
0.543& 0.182& 0.275\\
0.248& 0.375& 0.3776 
\end{pmatrix}}$, $\epsilon_p=0.3$ and $\epsilon_h=0.05$.
We observe that, at $t=3$ society stabilizes into two subgroups. This example shows that, as time evolves existing links can be destroyed as well (Figure~\ref{F2}).
\begin{figure}[H]
 \begin{minipage}{.2\textwidth}
  \includegraphics[scale=0.2]{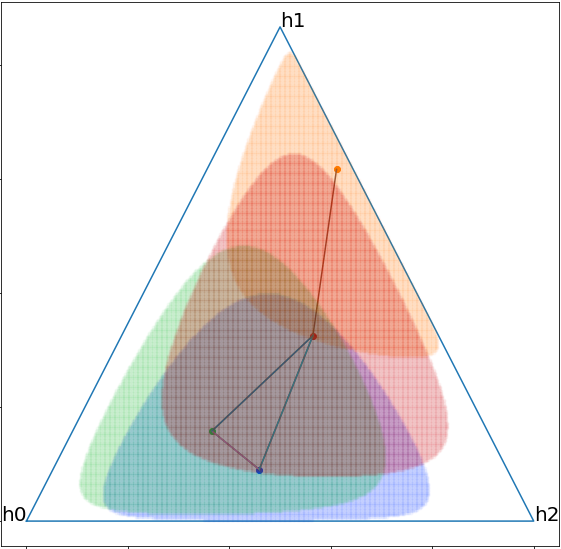}
   \caption*{t=1}
 \end{minipage}
  \hspace{10mm}
 \begin{minipage}{.2\textwidth}
  \includegraphics[scale=0.2]{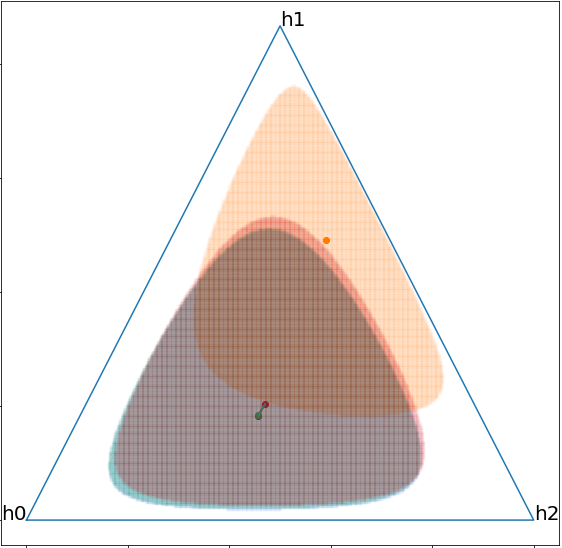}
\caption*{t=2}
 \end{minipage}
  \hspace{10mm}
 \begin{minipage}{.2\textwidth}
  \includegraphics[scale=0.2]{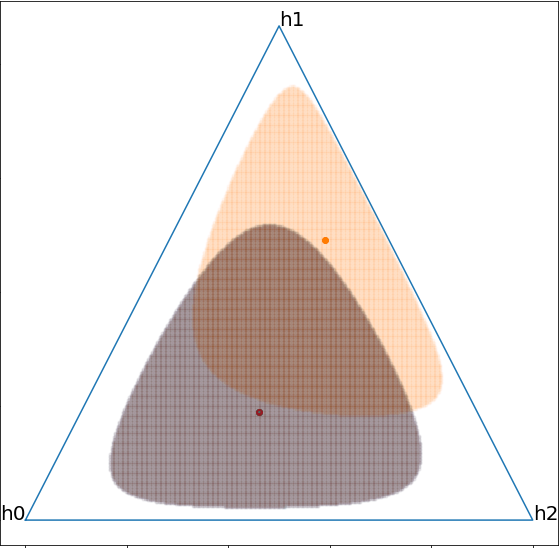}
\caption*{t=3}
 \end{minipage}
  \hspace{8mm}
 \begin{minipage}{.2\textwidth}
%  \hspace{15mm}
 \includegraphics[]{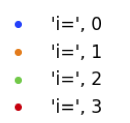}
%   \caption{t=5}
 \end{minipage} 
\caption{Each colored point inside a triangle represents the belief of a person $i$ at time $t$. The corresponding colored region represents the KL region of the person $i$. A line that connects two points denotes that the corresponding two people are communicating. } \label{F2}
\end{figure}
\end{example}
\begin{example}\label{Ex:KL3}
Now we consider $M$ and $\epsilon_p$ as in Example \ref{Ex:KL2} and let $\epsilon_h=0.5$.
We observe that, no links will be destroyed and at $t=3$ society stabilizes to one stationary belief distribution (Figure \ref{F3}). Observe that this example clearly illustrates the fact that the structure of the concept space affect the long term dynamics of belief distribution in a society, just as the social structure.
\begin{figure}[H]
 \begin{minipage}{.2\textwidth}
  \includegraphics[scale=0.2]{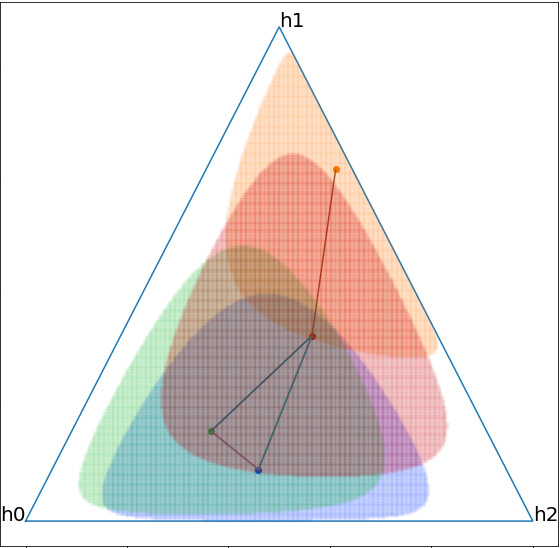}
   \caption*{t=1}
 \end{minipage}
  \hspace{10mm}
 \begin{minipage}{.2\textwidth}
  \includegraphics[scale=0.2]{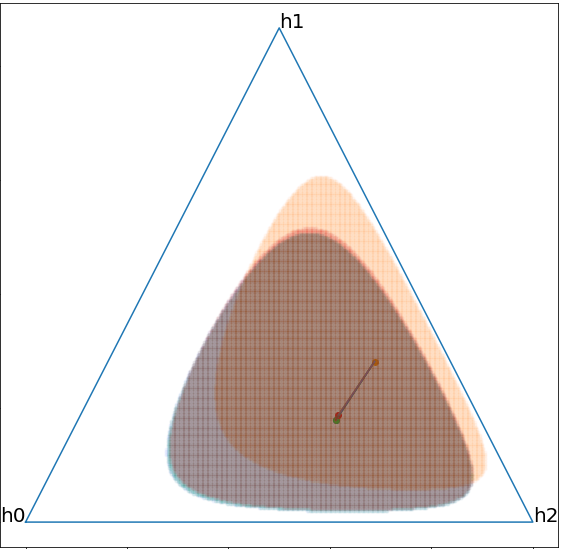}
\caption*{t=2}
 \end{minipage}
  \hspace{10mm}
 \begin{minipage}{.2\textwidth}
  \includegraphics[scale=0.2]{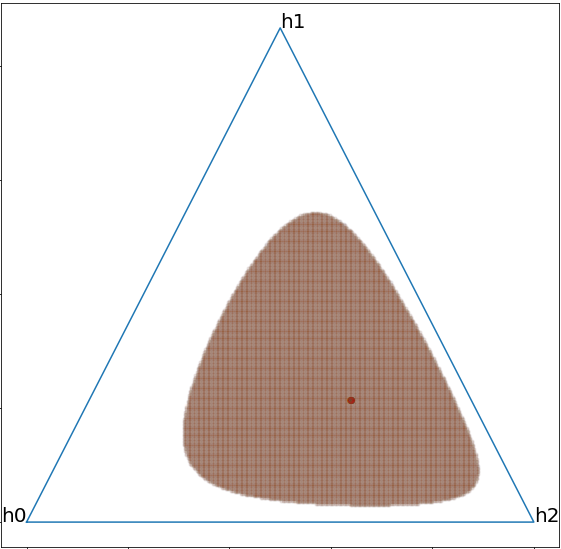}
\caption*{t=3}
 \end{minipage}
  \hspace{8mm}
 \begin{minipage}{.2\textwidth}
%  \hspace{15mm}
 \includegraphics[]{legend1.png}
%   \caption{t=5}
 \end{minipage} 
\caption{Each colored point inside a triangle represents the belief of a person $i$ at time $t$. The corresponding colored region represents the KL region of the person $i$. A line that connects two points denotes that the corresponding two people are communicating.} \label{F3}
\end{figure}
\end{example}

\section{Discussion}

We presented a mathematical model of that allows for transmission of beliefs over a set of concepts both across people (horizontal) and across time (vertical).  The model assumes structures over both individuals and concepts.  Individuals’ beliefs about a particular concept can change either because they are connected to an individual with different beliefs or because of a change in beliefs about a related concept.  We analyzed three cases: static social network and concept structures, social network and concept structures that change at random over time, and structures that vary dynamically based on homophily.

For static and randomly changing networks, we proved that if indecomposibility is satisfied by the initial (collection of) structures, then individuals in society will converge to a single group with the same beliefs. In the case of dynamically changing networks, we find a sufficient condition for heterogeneity to occur.  We also provided lower bounds for the rate of convergence of the model for both static and changing networks.  Our results align with previous studies showing rates of convergence slow with multidimensional transmission~\cite{page2007conformity}. For network structures that dynamically change based on homophily, we find that the society could either converge to a homogeneous distribution or sub groups with same beliefs and or to isolated individuals, based on a threshold on divergence between people and between beliefs. We proved conditions under which the lower bound on the number of groups is greater than one, thus identifying sufficient conditions under which individuals within society will converge to more than one group characterized by different beliefs.

Prior analyses of horizontal transmission have investigated richer social network structures, but have not considered learners who maintain distributions of beliefs.  This research has focused on rate of transmission as a function of the connectivity pattern in the graph. Transmission is assumed to occur by copying a random neighbor in the graph. For example, small world network structures \cite{albert2002statistical} yield rapid transmission to a large proportion of the network due to the short average minimal distance between individuals.  Thus, it does not allow for the possibility of polarization.

Our findings differ from prior analyses of vertical transmission which consider static network structures and chains of individuals  passing beliefs via random selection of data unidirectionally \cite{griffiths2007language} which show that convergence to a stationary distribution. This analysis holds for cases where individuals do not receive information from the world, and for cases where they receive data from both their predecessor and the world \cite{griffiths2005bayesian}. Across these cases, individuals in society, after long enough, all hold the same beliefs up to some variance that depends on the amount of data sampled from the world.

For example, \cite{whalen2017adding} considered vertical transmission of languages together with social structure.  In their model, at each timepoint, a random learner was paired with a random neighbor and heard their language, updating their own language probabilistically based on their prior and that observation. The primary findings were that the distribution of languages over the society converged to the prior and that  the degree to which neighbors in the graph spoke the same language depended on the social structure. Their study differed from ours in that they assumed each individual spoke only one language at a time, rather than maintaining a distribution and that individuals updated their language based on Bayesian inference. In contrast, we analyzed learners who maintained a distribution over beliefs and integrated information from prior timesteps with neighbors’ evidence based on information integration theory \cite{cohen1980information, frey1980information}.  Most important, though, by allowing for both networks of individuals and concepts to adapt, we enable the potential emergence of heterogeneity in beliefs through homophily.

Our results suggest that homophily based networks, which dynamically change to connect people with similar beliefs, yield stable heterogeneity; however, simpler arrangements in which changes in network structure over time are not related to beliefs do not. An implication of this work is to focus attention on homophily as a critical component in shaping stable, long term differences in beliefs that define communities.

% \pw{This paragraph is the one that addresses Jessica's concern.}

Evolution of beliefs is a type of collective learning in the absence of meaningful
feedback on any ground truth \cite{almaatouq2020adaptive}. Our model illustrates the importance of looking at vertical and horizontal transmission together: 
from a horizontal transmission perspective, any connected group of people converges to a society with a single belief distribution; 
while from a vertical transmission perspective, any connected structure leads to homogeneity convergence. 
When changes in horizontal structure accumulate over time because of homophily, we find stable heterogeneity.  
%Collective intelligence, which requires learning and incorporating other’s views, can be therefore be endangered by homophily. 
Collective intelligence requires differences in beliefs across individuals~\cite{march1991exploration,bednar2010emergent} and is enabled by homophily. However, collective intelligence is endangered by extremes of homophily in which one only talks with those of like beliefs. 

%in the extreme in which one only talks with 
%Some heterogeneity of course may be beneficial as it enables for exploration of new thinking and faster learning~\cite{march1991exploration,bednar2010emergent}.}

There remain a number of interesting open directions for future work including the death and birth of people and concepts, alternative models of transmission between neighbors, the possibility that people may obtain information from the environment, and the potential for dishonest actors who inject false information.   Experimental or empirical work could attempt to calibrate our models to behavioral data which could produce more realistic models of the horizontal and vertical evolution of beliefs and potentially bound rates of convergence.   Finally, our framework could be used to compare how the variation in the concept structure influences rates of convergence and possible to investigate the extent to which allocating concepts into disciplines impedes learning.

\section*{Acknowledgements}
This research was supported in part by DARPA grant HR00112020039, NSF MRI 1828528, and NSF Inspire 1549981 to PS.

\bibliographystyle{plain}
\bibliography{references}

\newpage

\appendix
\begin{center}

\textbf{ \Large Supplemental Material}\label{sec:Supp}
    
\end{center}
%\pu{remove unnecessary definitions and make  compact}

\section{Markov chain theory-definitions and preliminaries} \label{Sec:Def}

Below we summarize some definitions and preliminaries in Markov chain theory \cite{pishro2016introduction, feller1957introduction, gravner2010lecture}
\begin{definition} \label{Def:access}
Consider a Markov Chain with finite state space $X={1,2,....,N}$ and denote the transition matrix by $P=(p_{ij})$.
We say that state $j$ is accessible from state $i$, if $p_{ij}^n>0$ for some $P^{n}$. 
% Two states $i$ and $j$ are said to communicate, are accessible from each other. 
The states $i$ and $j$ belong to the
same communicating class if they are accessible from each other.
% In other words $p_{ij}>0$ and $p_{ji}>0$.
\end{definition}

% \begin{definition}
% A finite square matrix $P =\{p_{ij}\}$ is called stochastic if $p_{ij} \geq 0$ for all $i, j$, and $\sum_j p_{ij}=1$ for all $i$.
% \end{definition}

% \begin{definition} \label{def: transient,persistent}
% %Hitting time of state $i$ is the first return time to state $i$, namely,
% $T_i:=\inf\{n\ge 1: X_n=i\}$.
% State $i$ is called transient if, given that we start in state $i$, there is a non-zero probability that we will never return to $i$. %That is, state $i$ is transient if $ \Pr(T_{i}<\infty \mid X_{0}=i)<1.$
% State $i$ is called recurrent (or persistent) if it is not transient. %Recurrent states are guaranteed (with probability 1) to have a finite hitting time. 
% % State $i$ is called absorbing if and only if
% % $ p_{ii}=1{\text{ and }}p_{ij}=0{\text{ for }}i\not =j.$ That is, it is impossible to leave that state. 
% \end{definition}

\begin{definition} \label{def:period}
% A state $i$ is called periodic if the chain can return to the state only at multiples of some integer larger than 1. 
A state $i$ has period $d$ if any return to state $i$ occurs in multiples of $d$ time steps. That is, the period of a state $i$ is called $d$ if $d=\gcd \{n>0: Pr(X_n=i|X_0=i)>0\}$, where $\gcd$ is the greatest common divisor. 
If $Pr(X_n=i|X_0=i)=0$ for all $n>0$, then $d$ is $\infty$.
If $d=1$ then the state is called aperiodic. 
% For a state i let d(i) denote the greatest common divisor (gcd) of all integers l ≥ 1 such
% that P(l)ii > 0. If P(l)ii = 0 for all l ≥ 1, then we set d(i) = 0.
\end{definition}

% \begin{definition} \label{Def:closed}
% A set $C$ of states is closed if no state outside $C$ can be reached from any state %$X_j$ 
% $j$ in $C$.
% \end{definition}
% % \pw{a bit confusing on the notation for $X_n$ and $X_j$}
% \begin{definition} \label{def:irred}
% A Markov chain is irreducible if there exists no closed sets other than the set of all sets. Otherwise it is reducible.
% \end{definition}

% \begin{definition} \label{def:indecom}
% A Markov chain is indecomposable if it contains at most one closed set of states other than the set of all states.
% \end{definition}

\subsection{Graph theoretic interpretation} \label{sec:GT}

Next, we will rephrase some of the above definitions in terms of graph theory.
Associated with a finite state Markov chain of transition matrix $A_{r\times r}$, there is a directed graph $G_{A}$
with a vertex set $V= \{1, \dots, r\}$ and an edge set $E \subset V\times V$.
Each vertex corresponds to a state of $A$, and $(i,j)\in E$ if and only if $a_{ij} > 0$. 
A state $j$ is \textit{accessible} from state $i$, i.e. $i\to j$ $\Longleftrightarrow$ there exists a directed path from vertex $i$ to vertex $j$.
States $j$ and $i$ are \textit{communicate}, i.e. $i \leftrightarrow j $  $\Longleftrightarrow$ there exists a directed path from $i$ to $j$ 
and a directed path from $j$ to $i$.
A graph $G$ is called \textit{strongly connected} if there exists a directed path between any pairs of vertices of $G$.
Hence, \textit{communicating classes} of $A$ $\Longleftrightarrow$ maximal strongly connected components of $G_{A}$.
Moreover, each $G_{A}$ induces a condensed graph $\widehat{G}_{A}$ by combining vertices in each strong component into a `super-vertex'.
It is clear that $\widehat{G}_{A}$ must be \textit{acyclic} (no edge loops). 
Then \textit{recurrent states} of $A$ $\Longleftrightarrow$ states contained in leaf vertices of $\widehat{G}_{A}$ (or roots, depends on our direction of edges). 
%\pw{maybe move to the definition section?}

Corresponding to a set of possible transition matrices $\mathcal{S} = \{A_1, \dots, A_l\}$ (between a fixed collection of states), there is a set of directed graphs $\mathcal{S}_{\text{graph}} = \{G_{A_1}, \dots, G_{A_l}\}$. 
 Each graph in $\mathcal{S}_{\text{graph}}$ has exactly the same vertex set, whereas their edge sets may be different.
Denote the graph formed by the union of graphs in $\mathcal{S}_{\text{graph}}$ by $G_{\mathcal{S}}$, i.e. $G_{\mathcal{S}}$ has the same vertex set 
 as any $G_{A_k}$, and the edge set contains $(i,j)$ if there exists a $k \in \{1, \dots, l\}$ such that $G_{A_k}$ contains $(i,j)$.

For a time-inhomogeneous Markov chain, the state classification still make sense under the probability point of view:
\begin{definition} \label{def:GT}
Given a Markov chain whose transition matrix is sampled from $\s$ with respect to $\mathbf{w}$ for each step, 
a state $j$ is said to be \textit{accessible} from state $i$  
if there exists a finite product of matrices from $\s$, denoted by $C$, such that $c_{ij} >0$,
which is equivalent to the existence of a directed path from $i$ to $j$ in $G_{\mathcal{S}}$. 
In particular, state $j$ being accessible from state $i$ implies that positive transition from $i$ to $j$ occurs infinitely often 
with probability $1$ in any realization. (This is shown in the proof of Proposition~\ref{prop: prob_one_converge} where $C$ appears infinitely many times 
in any infinite product sampled from $\s$ with probability $1$).
%if $j$ is accessible from $i$ in an i.i.d. sampled product of transition matrices with probability $1$, 
Similarly the definition for \textit{communicating class} also generalizes.
Hence by combining vertices in the same class, we have $\widehat{G}_{\s}$.
A state is defined to be \textit{recurrent} if it is contained in a leaf of $\widehat{G}_{\s}$, otherwise the state is 
\textit{transient}. 
\end{definition}

Many critical features of a given Markov chain can be read off from its associated graph $\widehat{G}_{\s}$.
For example, $\widehat{G}_{\s}$ has more than one connected components suggests there are at least two sets of 
states never communicate, hence every matrix in $\s$ must be decomposable.
Moreover, assume $\widehat{G}_{\s}$ is connected and has more than one leaf.
If there exist two classes of recurrent states that are accessible from the same class of transient state, i.e. 
two leaf vertices in $\widehat{G}_{\s}$ have a common ancestor, then an i.i.d. sampled product diverges (not converge) with probability $1$ for many choices of $\s$.
In this case, recurrent states are eventually stabilized, but transient states are mixtures of recurrent states where the mixture weights varies as different 
transition matrices are sampled.

 \section{Time homogeneous Markov chains} %\pu{topics}

\begin{prop} \label{prop:period} \cite{{gravner2010lecture}}
If the Markov chain is indecomposable and has period $d$, then for every pair of states $i,j$ there exists an integer $r, \,\, 0\le r\le d-1$, such that
$$\lim_{k\to \infty}p_{ij}^{k d+r}=d \, \pi_j$$ and $p_{ij}^n=0$ for all $n$ such that $n \neq r \mod d$.
\end{prop}
% \pw{what dose this prop mean? $d \pi_j$ means $d \cdot \pi_j$? }  

\begin{prop} \label{prop:decom,aperiod} \cite{{gravner2010lecture}}
Suppose the Markov chain is decomposable and aperiodic. Then $\bm{\pi}$ is not unique. In particular, 
$$\lim_{n\to \infty}p_{ij}^n=h_i^C \pi_j$$ where $h_i^C$ denotes the hitting probability of the closed class $C$ with $j\in C$ starting from state $i$.
\end{prop}
% \pw{a bit confused by this prop too}

If the chain is decomposable and periodic use propositions \ref{prop:period} and \ref{prop:decom,aperiod}.

\begin{proof} [\textbf{Proof of Proposition~\ref{prop:homogeneous}}] \label{proof:homogeneous}

\begin{enumerate}[(i)]
    \item  \label{1}
    If $H$ is indecomposable and aperiodic, then $H$ has a single stationary distribution $\bm{\pi}=\{\pi_1,....,\pi_s\}$ \cite{serfozo2009basics}. %such that $$\lim_{n \to \infty} H^n = \bm 1 \bm \pi$$ where $\bm 1$ is the column vector with all entries equal to 1.
%From proposition \ref{prop:same rows}, there exists a steady state distribution $\bm{\pi}=\{\pi_1,....,\pi_s\}$ for $H$. %which is a left eigenvector of $H$ with eigenvalue equals $1$. 
Moreover $\lim_{n\to \infty} P^n M$ is a stochastic matrix (product of two stocastic matrices is a stochastic matrix). Denote $\lim_{n\to \infty} P^n M=(\alpha_{ij})_{r \times s}.$
Then
\begin{align*}
\lim_{n\to \infty}Q_n=P^nMH^n
&=
\begin{pmatrix}
\alpha_{11} & ...& \alpha_{1s}\\
\vdots & \ddots & \vdots \\
\alpha_{r1} & ...& \alpha_{rs}\\
\end{pmatrix}_{(r\times s)} \cdot
\begin{pmatrix}
\pi_1 & ...& \pi_s\\
\vdots & \ddots & \vdots \\
\pi_1 & ...& \pi_s\\
\end{pmatrix}_{(s\times s)} \\
&=
\begin{pmatrix}
\pi_1 & ...& \pi_s\\
\vdots & \ddots & \vdots \\
\pi_1 & ...& \pi_s\\
\end{pmatrix}_{(r\times s)}
%=\lim_{n\to \infty}H^n
\end{align*}
since $ \sum_{j=1}^s \alpha_{ij}=1, \forall i.$

\item \label{2}
Proof follows by similar argument to (\ref{1}). 
\item
Proof can be obtained by Proposition \ref{prop:decom,aperiod} and (\ref{2}).
\end{enumerate}
\end{proof}
% \pu{example} reducible-transient vanish

% \section*{Time homogeneous Markov chains - rate of convergence}
 Let the eigenvalues of $A$ (counted with algebraic multiplicity) be $\lambda_0, \lambda_1,....,\lambda_{n-1}.$ Without loss of generality take $\lambda_0=1$ and
 set $\lambda_*=\max_{1\leq j\leq n-1} |\lambda_j|$. Then $\lambda_* \leq 1.$ Moreover, if $A$ is indecomposable then $|\lambda_*|<1$.

\begin{definition}\label{Def:convergence rate} \cite{schatzman2002numerical} \textbf{[Rate of convergence] 
 A sequence $ \{x_{n}\}$ that converges to $x^{*}$ is said to have order of convergence $q\geq 1$  and rate of convergence  $\mu$,  if
  $\lim _{n\rightarrow \infty } {\frac {\left|x_{n+1}-x^{*}\right|}
  {\left|x_{n}-x^{*}\right|^{q}}}=\mu$
}
\end{definition}

\begin{prop}\cite{rosenthal1995minorization} \label{prop:EW}
An indecomposable and aperiodic Markov chain converges to its stationary distribution geometrically quickly. 
In particular, if $P$ is indecomposable and aperiodic then there exists a positive constant $C$ such that for all $i,j=1,...,N$ $$|p^n_{ij}- \pi_j|\leq C \lambda_*^n$$
where $\bm \pi=\{\pi_1,...,\pi_N\}$ is the stationary distribution. %\pu{change-ith row} %\bm 1 \pi
\end{prop}

\section{Time inhomogeneous Markov chains}\label{apd:inhomo}
\begin{prop}  \label{prop:products_of_A}
\cite{chevalier2017sets} 
Let $S$ be a finite set of stochastic matrices of the same order.
Any product of matrices from $S$ converges to a rank one matrix if and only if every product of matrices in $S$ 
%(with repetitions allowed)
is SIA.
\end{prop}

\begin{prop}\label{lemma:scrambling} \cite{wolfowitz1963products}
If one or more matrices in a product of matrices is scrambling, so is the product. 
\end{prop} 
\begin{prop} \cite{seneta2006non} \label{prop:scrambling implies SIA}
Any stochastic scrambling matrix is SIA.
\end{prop}

\begin{proof} [\textbf{Proof of \textbf{Proposition~\ref{prop:rankone_inhomo}}}] \label{proof:rankone_inhomo}
As a direct application of Proposition~\ref{prop:products_of_A},
the assumption that every product of matrices in $\s_{P}$ or/and $\s_{H}$ (with repetitions allowed) is SIA 
implies either $\prod_{t=0}^{T} P_t$ or $\prod_{t=0}^{T} H_t$ converges to a rank one matrix as $T \to \infty$.
Hence by proposition \ref{prop:homogeneous}, $\prod_{t=0}^{T} P_t  M \prod_{t=0}^{T} H_t$ converges to a rank one matrix as $T \to \infty$.
\end{proof}

\begin{proof}[\textbf{Proof of \textbf{Proposition~\ref{prop:product_rankone}}}] \label{proof:product_rankone}
According to Proposition~\ref{prop:products_of_A}, any product of matrices from $\s$ converges to a rank one matrix if and only if every product of matrices in $\s$ 
is SIA. Hence we only need to check that every product of matrices in $\s$ is SIA. 
Since product of stochastic matrices is still stochastic, 
Proposition~\ref{prop:scrambling implies SIA} - any stochastic scrambling matrix is SIA indicates that we only need to show that every product of matrices in $\s$ is 
scrambling. And this holds as Proposition~\ref{lemma:scrambling} shows that if one or more matrices in a product of matrices is scrambling, so is the product. 
Thus we are done.
\end{proof}

\begin{definition}
\begin{align*}
\delta(P)=\max_j \max_{i_1,i_2}|p_{i_1j}-p_{i_2j}|.    
\end{align*}
\end{definition}
Thus $\delta(P)$ measures, in a certain sense, how different the rows of $P$ are. If the rows of P are identical, $\delta(P) =0$ and conversely.

\begin{proof} [\textbf{Proof of \textbf{Proposition~\ref{prop: prob_one_converge}}}] %\label{proof:prob_one_converge}
The only if direction: let $B^1, B^2, \dots$ be a sequence formed by the i.i.d. fashion as described above that converges to a rank one matrix.
Then according to Definition~\ref{def:scrambling}, $\lambda (B^1B^2\dots B^k)$ converges to $0$ as $k\to \infty$. Thus, there exists $N \in \mathbb{Z}^+$ such that $\lambda (B^1B^2\dots B^N)<1$, and so $B^1B^2\dots B^N$ is scrambling.

The if direction: let $C^1C^2\dots C^k$ be a finite product of matrices from $\mathcal{S}$ that is scrambling, in particular, $\lambda(C^1C^2\dots C^k) <1$. Note that $C^j = A_{i_j}$ for each $j$, where $i_j$ is an integer in $\{1,\dots, l\}$. 
Then, for an i.i.d sampled sequence $\{B^1, B^2, \dots, B^k\}$ of length $k$, the portability 
$\mathcal{P}(\prod_{j=1}^{j=k} B^j \neq \prod_{j=1}^{j=k} C^j) =1- \prod_{j=1}^{j=k} w_{i_j}$  is less than 1.
Hence the probability of $C^1C^2\dots C^k$ appears infinitely many times in a infinite sequence $\{B^1, B^2, \dots \}$ is $1$.
Thus with probability $1$, $\delta(\prod_{j=1}^{\infty} B_j) \leq [\lambda(C^1C^2\dots C^k)]^{N}$ for any given $N$.
Note that $\delta(\prod_{j=1}^{\infty} B_j) \leq [\lambda(C^1C^2\dots C^k)]^{N} \to 0$ as $N \to \infty$.
This implies that $\{B^1, B^2, \dots \}$ converges to a rank matrix. 

\end{proof}

% \begin{proof} [\textbf{Proof of Proposition~\ref{prop: prob_one_converge}}] \label{proof:prob_one_converge}
% The only if direction: let $B^1, B^2, \dots$ be a sequence formed by the i.i.d. fashion as described above that converges to a rank one matrix.
% Then according to Definition~\ref{def:scrambling}, $\lambda (B^1B^2\dots B^k)$ converges to $0$ as $k\to \infty$. Thus, there exists $N \in \mathbb{Z}^+$ such that $\lambda (B^1B^2\dots B^N)<1$, and so $B^1B^2\dots B^N$ is scrambling.

% The if direction: let $C^1C^2\dots C^k$ be a finite product of matrices from $\mathcal{S}$ that is scrambling, in particular, $\lambda(C^1C^2\dots C^k) <1$. Note that $C^j = A_{i_j}$ for each $j$, where $i_j$ is an integer in $\{1,\dots, l\}$. 
% Then, for an i.i.d sampled sequence $\{B^1, B^2, \dots, B^k\}$ of length $k$, the portability 
% $\mathcal{P}(\prod_{j=1}^{j=k} B^j \neq \prod_{j=1}^{j=k} C^j) =1- \prod_{j=1}^{j=k} w_{i_j}$  is less than 1.
% Hence the probability of $C^1C^2\dots C^k$ appears infinitely many times in a infinite sequence $\{B^1, B^2, \dots \}$ is $1$.
% Thus with probability $1$, $\delta(\prod_{j=1}^{\infty} B_j) \leq [\lambda(C^1C^2\dots C^k)]^{N}$ for any given $N$.
% Note that $\delta(\prod_{j=1}^{\infty} B_j) \leq [\lambda(C^1C^2\dots C^k)]^{N} \to 0$ as $N \to \infty$.
% This implies that $\{B^1, B^2, \dots \}$ converges to a rank matrix. 
% \pw{double check if $\delta$ is defined before.}
% \end{proof}

\begin{proof} [\textbf{Proof of Corollary~\ref{cor: one_leaf}}] \label{proof: one_leaf}
According to Proposition~\ref{prop: prob_one_converge}, we only need to show that (1) $\widehat{G}_{\s}$ is connected and has one leaf is equivalent to 
(2) a finite product of matrices from $\mathcal{S}$ is scrambling.

$(2)$ $\Longrightarrow$ $(1)$: We will prove by contradiction.
It is clear that $\widehat{G}_{\s}$ must be connected, otherwise all product of matrices from $\s$ must be decomposable.
Now assume that $\widehat{G}_{\s}$ is connected but has more than one leaf, and let $S_1$ and $S_2$ be sets of states contained in two different leaf vertices. 
Then any directed edge path starts from $s_i\in S_i$ must terminate at a vertex in $S_i$ for $i= 1,2$.
Hence for any finite product $C$ from $\s$, $c_{s_i, j} = 0$ if $j \notin S_i$.
In particular, this implies that row $s_1$ and row $s_2$ of $C$ have positive elements in different columns.
Therefore ergodic coefficient of $C$ :  $\gamma(C) = \min_{i_1, i_2} \sum_j \min (c_{i_1, j}, c_{i_2, j}) = \sum_j \min (c_{s_1, j}, c_{s_2, j}) =0 $, 
and $\lambda(C) = 1- \gamma(C) =1$ $\Longrightarrow$ $C$ is not scrambling. This is contradict to $(2)$.

$(1)$ $\Longrightarrow$ $(2)$: We will prove by constructing a scrambling finite product.
Denote the states contained in the leaf by $S_1$ and all the other states by $S_2$.
Based on any two states in $S_1$ are communicate, it is easy to check that there exists a product $C$ of $\s$ such that 
$c_{i,j} > 0 $ for any $i,j \in S_1$. For any $k \in S_2$, since $\widehat{G}_{\s}$ is connected, there exists a directed path from $k$ to $S_1$. 
Hence, there exists a product $D^k$ such that $d^k_{k,j} >0$ for some $j\in S_1$.
Further note that for each $i\in S_1$, $d^k_{i,j} >0$ must hold for some $j\in S_1$ as $S_1$ is a leaf.
Combining the above features of $D^k$ and $C$, one may check that the product $E=(\prod_{k\in S_2} D^k )C$ satisfies that 
$e_{ij} > 0$ for any $j\in S_1$, which indicates that $\lambda(E) < 1$, i.e. $E$ is scrambling.
\end{proof}

\begin{proof}[\textbf{Proof of Corollary~ \ref{prop: expectation}}] \label{proof: expectation}
At each time $t$, the transition matrix is a random variable, denoted by $X^t$ as before. 
Then the expected transition matrix is $\avg (X^t) = \sum_{k=1}^{k=l} w_k \cdot A_k \triangleq \bar{A}$.
Hence, the expectation of an i.i.d. sampled product of length $N$ is $\avg (X^1\dots X^N) = \avg(X^1) \dots \avg(X^N) = \bar{A}^N \to \bar{A}^{\infty} $, 
as $N \to \infty$. $\bar{A}^{\infty} $ exists since $\bar{A}$ is a stochastic matrix.
\end{proof}

% \begin{definition}\textbf{Kullback–Leibler divergence.} \label{Def:KL}

% For discrete probability distributions $P$ and $Q$ defined on the same probability space $\mathcal {X}$, the Kullback–Leibler divergence from $Q$ to $P$ is defined to be
% $$KL(P,Q)=\sum _{x \in \mathcal{X}}P(x)\log \left({\frac {P(x)}{Q(x)}}\right)$$
% \end{definition}

\section{Computational framework: dynamic, homophily-based networks}

% \begin{alg} \label{alg1}

% \noindent Input: $\epsilon_p=$, threshold for $P$,
% $\epsilon_h=$, % threshold for $H$,
% $N=$max number of steps, $r=$ number of people, $s=$ number of concepts.

% \noindent Output: $P_t, Q_t$

% \noindent Initialize $M_{r\times s}$ using a Dirichlet distribution, $P_0=I_{r \times r}, H_0=I_{s\times s}, Q_0=P_0 M H_0$

% \noindent For t from 1 to N calculate:

% $KL^p_{mat}=(kl^p_{ij})$ such that \pu{change ($KL(\alpha_i,\alpha_j) to KL(p_i,p_j)$). Similarly for $\beta$}

% $kl^p_{ij}=KL(\alpha_i,\alpha_j)=\sum_{\beta_k \in \mathcal H} m_{ik} \log\left(\frac{m_{ik}}{m_{jk}}\right),$
    
% % $SM_{mat}=(SM_{ij})$ such that $SM_{ij}= \sigma(kl_{ij})$
    
% $P_t=(p^t_{ij})$ where
%  $p^t_{ij}=
%   \begin{cases}
%       \sigma(kl^p_{ij}), & \text{if}\,\,\,  kl^p_{ij}<\epsilon_p \\
%       0, & \text{otherwise}
%     \end{cases}$

% $P_t=$row normalize $P_t$

% \pu{$\widehat{M}$}
% \pu{$S_i,w_i$, etc}
% $\Tilde{M}$=column normalize $M$

% $KL^h_{mat}=(kl^h_{ij})$ such that $kl^h_{ij}=KL(\beta_i,\beta_j)=\sum_{\alpha_k \in \mathcal P} \tilde{m}_{ki} \log\left(\frac{\tilde{m}_{k_i}}{\tilde{m}_{kj}}\right),$

% $H_t=(h^t_{ij})$ where
%  $h^t_{ij}=
%   \begin{cases}
%     %   \frac{1}{\sigma(kl^h_{ij})}, & \text{if}\,\,\,  kl^h_{ij}<\epsilon_h \\
%      \sigma(kl^h_{ij}), & \text{if}\,\,\,  kl^h_{ij}<\epsilon_h \\
%       0, & \text{otherwise}
%     \end{cases}$
% % \pw{Again, it feels that the softmax for $H_t$ does not need a reciprocal. }

% $H_t=$row normalize $H_t$

% $Q_t=P_tQ_{t-1}H_t$

% $M=Q_t$

% \end{alg}

\begin{algorithm}[H]
\caption{Homophily based networks}
Inputs: $\epsilon_p=$ threshold for $P$,
$\epsilon_h=$  threshold for $H$,
$N=$max number of steps, $r=$ number of people, $s=$ number of concepts.

Initialize: $M=(m_{ij})_{r\times s}$ using a Dirichlet distribution

$P_0=I_{r \times r}, H_0=I_{s\times s}, Q_0=P_0 M H_0$

\For{t = 1 to N}{

$KL(\mathbf{p}_i,\mathbf{p}_j)=\sum_{\beta_k \in \mathcal H} m_{ik} \log\left(\frac{m_{ik}}{m_{jk}}\right),$
    
$P_t=(p_{ij})$ where
 $p_{ij}=
   \begin{cases}
      \sigma(KL(\mathbf{p}_i,\mathbf{p}_j)), & \text{if}\,\,\,  KL(\mathbf{p}_i,\mathbf{p}_j)<\epsilon_p \\
      0, & \text{otherwise}
    \end{cases}$

$P_t=$row normalize $P_t$

$\widehat{M}=(\hat m_{ij})$=column normalize $M$

% $KL^h_{mat}=(kl^h_{ij})$ such that $kl^h_{ij}=KL(\beta_i,\beta_j)=
$KL(\mathbf{h}_i,\mathbf{h}_j)=\sum_{\alpha_k \in \mathcal P} \hat{m}_{ki} \log\left(\frac{\hat{m}_{k_i}}{\hat{m}_{kj}}\right),$

$H_t=(h_{ij})$ where
 $h_{ij}=
   \begin{cases}
    %   \frac{1}{\sigma(kl^h_{ij})}, & \text{if}\,\,\,  kl^h_{ij}<\epsilon_h \\
     \sigma(KL(\mathbf{h}_i,\mathbf{h}_j)), & \text{if}\,\,\,  KL(\mathbf{h}_i,\mathbf{h}_j)<\epsilon_h \\
      0, & \text{otherwise}
    \end{cases}$
% \pw{Again, it feels that the softmax for $H_t$ does not need a reciprocal. }

$H_t=$row normalize $H_t$

$Q_t=P_tQ_{t-1}H_t$

$M=Q_t$

    %     $D_1 = ones(N_r)$ \;
    %     $D_{1, k}\leftarrow e^{\mathcal{N}(0, \sigma)}$ \;
    %     $D_{2,l}\leftarrow$ 1 / $\sum_{k=0}^{N_r} (diag(D_1)$  $M_i$)$_{kl}$ for $l$ = 1 to $N_r$\;
    %     $M_{i+1} \leftarrow diag(D_1)$ $M_i$ $diag(D_2)$\;
    %     ratio $\leftarrow Dir(M_{i+1}, \alpha) / Dir(M_i, \alpha)$\;
    %     \eIf{ratio > rv}{
    %         Accept $M_{i+1}$\;
    %     }{
    %         Reject $M_{i+1}$\;
    %         $M_{i+1} \leftarrow M_{i}$
    % }
}
\end{algorithm}